\documentclass[sigconf, nonacm]{acmart}

\AtBeginDocument{%
  }

\usepackage{pgfplots}
\pgfplotsset{compat=1.18}

\usepackage{tikz}
\usetikzlibrary{positioning, arrows.meta, fit, shapes.geometric, calc}

\newcommand{\stagelead}[1]{\smallskip\noindent\textbf{#1}}

\setcopyright{none}
\begin{document}

\title{FORTE: A Production Framework for Feature Evaluation and Deprecation
  in Industrial Recommender Systems}

\author{Vishnu Arun}
\affiliation{%
  \institution{Pinterest}
  \city{San Francisco}
  \state{CA}
  \country{USA}
}
\email{vishnuarun@pinterest.com}

\author{Adam Obeng}
\affiliation{%
  \institution{Pinterest}
  \city{San Francisco}
  \state{CA}
  \country{USA}
}
\email{aobeng@pinterest.com}

\author{Tianyou Wang*}
\affiliation{%
  \institution{Pinterest}
  \city{San Francisco}
  \state{CA}
  \country{USA}
}
\email{tianyouwang@pinterest.com}

\author{Zilin Ren}
\affiliation{%
  \institution{Pinterest}
  \city{San Francisco}
  \state{CA}
  \country{USA}
}
\email{zren@pinterest.com}

\author{Mukuntha Narayanan S}
\affiliation{%
  \institution{Pinterest}
  \city{San Francisco}
  \state{CA}
  \country{USA}
}
\email{mukuntha@pinterest.com}

\author{Han Wang}
\affiliation{%
  \institution{Pinterest}
  \city{San Francisco}
  \state{CA}
  \country{USA}
}
\email{hanwang@pinterest.com}

\author{Shivin Thukral}
\affiliation{%
  \institution{Pinterest}
  \city{San Francisco}
  \state{CA}
  \country{USA}
}
\email{sthukral@pinterest.com}

\author{Zhixiang Wang}
\affiliation{%
  \institution{Pinterest}
  \city{San Francisco}
  \state{CA}
  \country{USA}
}
\email{zhixiangwang@pinterest.com}

\author{Saurabh Vishwas Joshi}
\authornote{Work done while at Pinterest}
\affiliation{%
  \institution{Pinterest}
  \city{San Francisco}
  \state{CA}
  \country{USA}
}
\email{saurabhsvj@gmail.com}

\author{Darren Reger}
\affiliation{%
  \institution{Pinterest}
  \city{San Francisco}
  \state{CA}
  \country{USA}
}
\email{dreger@pinterest.com}

\renewcommand{\shortauthors}{Arun et al.}

\begin{abstract}
Production recommender systems have feature-lifecycle asymmetry: adding new
features need only produce gains on some metric, while feature removal must not
regress any metric. In practice, feature deprecation is hard to prioritize as a
feature can easily fail to be removed: if it interacts with other features, its
importance changes daily, it cannot be evaluated before production models change
or data drift occurs, it is in contention for A/B test space with other
launches, it causes regressions in an online A/B test, or it is consumed by
multiple models that need to coordinate its deprecation.

We present FORTE, a model-agnostic production framework for discovering,
validating, and removing features. Its efficient perturbation implementation
allows daily feature evaluation. To account for interactions, it evaluates
candidate feature sets constructed from features with the weakest measured
importance through
grouped perturbation, providing an offline gate before online testing.
Lineage tracing identifies all downstream consumer models to realize the full
savings of removing a feature from the training set.

We validate FORTE over 2 large-scale production recommender systems and find that it
identifies deprecable feature sets which do not cause online regressions, while
the impact on efficiency depends on model characteristics. On a search-relevance
model, FORTE resulted in 3 nested candidate sets of 32, 48, and 59 features.
During offline evaluation and retraining, all 3 candidates appeared neutral, but
online guardrails rejected the largest treatment. The 48-feature set was
selected for deprecation, removing 43.2\% of the model's features. On a large
ranking model, FORTE identified 8 candidates for removal. Retraining and A/B
testing confirmed no regression. Its launch reduced the trainable parameters by
29.9\%, improved serving throughput by 19.8--25.5\%, and reduced annual training
costs by hundreds of thousands of dollars.
\end{abstract}

\ccsdesc[500]{Information systems~Recommender systems}
\ccsdesc[300]{Computing methodologies~Machine learning}

\keywords{Feature Deprecation, Feature Importance, Recommender Systems,
  Production Machine Learning, Distributed Inference}

\maketitle

\section{Introduction}

\subsection{Feature-lifecycle asymmetry}

Modern recommender systems drive engagement and revenue across billions of
daily interactions. Pinterest deploys hundreds of models for retrieval,
ranking, and engagement prediction, each ingesting tens to thousands of
features.

Model evolution requires experimenting with dozens of features at any given time. Adding a new feature is generally justified by a gain in a particular user behavioral metric, assuming
adherence to a strict set of model performance guardrails. Removing one is instead justified
by an efficiency gain, and its criteria are harder to satisfy: beyond the same
guardrails, deprecation generally requires validating that no user metric regresses.

Validating a feature removal requires both retraining a model and running an online A/B test of the treatment against the current production model. A/B tests are a limited resource and generally prioritize high-impact treatments, which feature removals do not promise. Deprecation
also creates a combinatorial allocation problem: a model with $d$ features has
$2^d-1$ nonempty candidate removal sets. Although only a small fraction will be
plausible, this candidate space is substantially larger than the set of new
feature launches competing for the same A/B capacity. This incentive structure
tends to lead to the persistence of ineffective features.

\subsection{Deprecation in practice}
\label{subsec:challenges}

In industry, model development typically follows the same iteration cycle for both adding and removing
a feature: propose the change, retrain a model with it enabled, backtest offline
on evaluation data, run an online A/B test, validate online user behavioral metrics and model performance
guardrails against the baseline, then audit and ship. As a change progresses
through the cycle, the associated costs increase as resources constrain.

Deprecation only realizes its benefits after the end of this cycle: removing a feature
reduces trainable parameters, producing a leaner model with lower training and
serving cost, and once no model references the feature it can stop being
collected in training and evaluation sets altogether.

We define deprecation specifically as removing a feature from a model which is already deployed. This
differs from the related concept of feature selection, which is typically
performed during model development to identify and remove low-signal features to
improve model generalization performance.

We identify the following failure modes for feature deprecation in an industry
setting. Each can prevent the feature-deprecation cycle from completing, cause
offline validation results to diverge from online outcomes, or both:

\begin{itemize}
  \item Feature interactions.
  \item Data drift.
  \item Model versioning and drift.
  \item Online A/B bandwidth contention.
  \item Online regression.
  \item Multi-model coordination during deprecation.
\end{itemize}

\subsection{FORTE overview}

We identify the predominant bottleneck to deprecation as limited online A/B bandwidth.
FORTE introduces a perturbation-based feature importance system
which is sufficiently efficient to run perturbation experiments on a daily
cadence, generates and ablates candidate groups, and traces feature lineage to
provide end-to-end feature removal.

\subsection{Contributions}

\begin{itemize}
  \item \textbf{Problem formulation.} We define complete feature deprecation in
    two parts: (1) the ability to identify unimportant feature sets from a
    model, and (2) tracking a feature's complete lineage and deprecating it once
    it is deemed unimportant.
  \item \textbf{FORTE, an end-to-end framework spanning discovery, validation,
    and removal.} We combine temporal feature importance, pooled perturbation,
    online A/B validation, and feature lineage tracking into a single workflow.
  \item \textbf{Efficient perturbation implementation.} We define an efficient
    perturbation implementation with a minimal marginal cost per additional
    perturbation experiment. This enables running individual-feature perturbation experiments
    daily.
  \item \textbf{Case studies from two production recommender systems.} We walk
    through a Search Relevance model and a Pin Ranker model.
\end{itemize}

\section{Related Work}

Permutation importance is a model-agnostic way to measure model reliance.
Breiman \cite{breiman2001} introduced it for Random Forests, and Fisher, Rudin,
and Dominici \cite{fisher2019} generalized it to Model Class Reliance. These
estimators match our operational question: how much does model quality change
when a feature's signal is destroyed? Existing implementations evaluate features
sequentially and lack native support for heterogeneous tensor types found in
production ML models.

The main statistical limitation of marginal permutation is feature dependence.
Unconditional permutation can create out-of-distribution inputs and distort
measured importance for correlated predictors \cite{hooker2021, nicodemus2010}.
Conditional PFI \cite{strobl2008} addresses this by conditioning on the other
predictors when permuting, but requires estimating conditional structure.
Covert, Lundberg, and Lee \cite{covert2021removing} distinguish fixed-value
removal from distributional marginalization; the latter may use marginal or
conditional replacements. FORTE's ablation uses fixed nulls, while permutation
uses batch-shuffled values as empirical marginal replacements. We do not claim
conditional importance, instead using online A/B validation for feature removal.

Stronger attribution alternatives are prohibitive at daily cadence. LOCO
\cite{lei2018} retrains without each feature, requiring $p$ retrainings. ROAR
\cite{hooker2019roar} evaluates rankings by cumulative removal and retraining;
its findings motivate external validation. SHAP \cite{lundberg2017}, KernelSHAP
\cite{covert2021kernelshap}, and SAGE \cite{covert2020} use Shapley-style
coalitions, but sampling cost is prohibitive at thousands of features. TreeSHAP
\cite{lundberg2018tree} cuts this cost for tree ensembles but does not apply to
the neural rankers we study.

Backpropagation-based attribution methods such as Integrated Gradients
\cite{sundararajan2017} and DeepLIFT \cite{shrikumar2017} require access to model internals and a reference input; baseline choice can materially affect the
resulting attribution \cite{sturmfels2020baselines}. FORTE targets global,
perturbation-based importance.

Streaming and online importance estimation \cite{fumagalli2022,
muschalik2023} update with each sample; FORTE uses scheduled model evaluations
for operational monitoring.

Unmanaged feature dependencies contribute to ML technical debt
\cite{sculley2015}. Before FORTE, feature importance at Pinterest required one
inference job per feature, making daily evaluation infeasible. FORTE combines
pooled inference, type-aware perturbation, temporal tracking, and online
validation in one workflow.

\section{Problem Formulation}
\label{sec:problem}

\stagelead{Terminology and sign convention.} In an offline perturbation
experiment, the \emph{baseline condition} is the frozen model evaluated on the
intact examples, and the \emph{perturbed condition} is the same model evaluated
after a feature group $G$ has been permuted or ablated. Let $M_0$ and $M_G$
denote the resulting baseline and perturbed metrics. FORTE reports the
direction-adjusted normalized change
\begin{equation}
  \label{eq:normalized-change}
  \Delta_M(G)
  = 100\,s_M \cdot \frac{M_G-M_0}{M_0}.
\end{equation}
Here $s_M=+1$ when higher values of $M$ are better and $s_M=-1$ when lower
values are better.
We refer to this signed normalized change as the \emph{feature-importance
score}. Thus, a negative score means that perturbing $G$ degrades model quality
and the model relies on the feature group; a score near zero indicates weak
measured reliance; and a positive score means that perturbation improves model
quality. Under this convention, scores increase as measured reliance decreases,
so deprecation candidates are the features with scores near zero or above.
We report a signed change rather than a metric drop so that the score carries
the same orientation as the online A/B results it is compared against.
Classical permutation importance \cite{breiman2001, fisher2019} reports the
negated quantity, so reliance appears there as a positive degradation; readers
should keep this difference in mind when comparing magnitudes across the two
conventions.

In the retraining and online stages, the \emph{baseline} (or online
\emph{control}) is the incumbent production model, while the \emph{treatment}
is a model retrained with a candidate set removed. To avoid overloading these
terms we separate offline from online work throughout. A
\emph{perturbation experiment} is one configured offline comparison of the
baseline and perturbed conditions on a frozen model, and we report offline
results per experiment. A \emph{treatment arm} is a retrained candidate model
allocated online traffic in an \emph{A/B test}, and we report online results per
arm; arms are the scarce resource the funnel of
Section~\ref{sec:system} is ordered around. A \emph{candidate set} is a
collection of features proposed for joint removal.

Given a production model $f$ with feature set $\{X_1, \ldots, X_d\}$ and
evaluation metric $M$, we aim to:

\begin{enumerate}
  \item Compute daily feature importance scores $\mathrm{FI}_j(t)$ for all
    features $j$ at each day $t$, producing temporal series
    $\mathrm{TFI}_j = \{\mathrm{FI}_j(1), \ldots, \mathrm{FI}_j(T)\}$.
  \item Combine temporal importance into a ranked
    deprecation candidate set; validate via retraining and online A/B
    testing.
  \item Identify feature lineage and create a method for safe deprecation.
\end{enumerate}

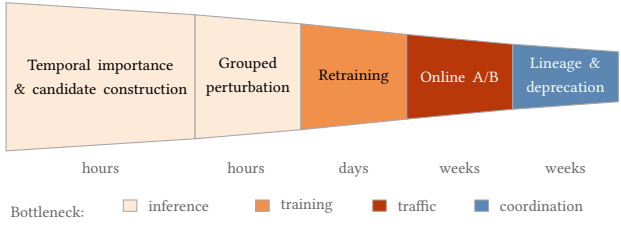
\begin{figure}[t]
  \centering
  \begin{tikzpicture}[x=1cm, y=1cm, font=\small]

    \definecolor{fInf}{HTML}{FDEAD9}
    \definecolor{fTrain}{HTML}{F0904A}
    \definecolor{fOnline}{HTML}{B8380A}
    \definecolor{fCoord}{HTML}{5B87B2}

    \foreach \xa/\xb/\ha/\hb/\col in {
        0/2.50/0.98/0.82/fInf,
        2.50/3.90/0.82/0.70/fInf,
        3.90/5.30/0.70/0.56/fTrain,
        5.30/6.70/0.56/0.43/fOnline,
        6.70/8.10/0.43/0.33/fCoord} {
      \fill[\col, draw=black!35, line width=0.4pt]
        (\xa,\ha) -- (\xb,\hb) -- (\xb,-\hb) -- (\xa,-\ha) -- cycle;
    }

    \node[align=center, text width=2.35cm] at (1.25,0)
      {\scriptsize Temporal importance\\[-1pt]
       \scriptsize \& candidate construction};
    \node[align=center, text width=1.25cm] at (3.20,0)
      {\scriptsize Grouped\\[-1pt]\scriptsize perturbation};
    \node[align=center, text width=1.25cm] at (4.60,0)
      {\scriptsize Retraining};
    \node[align=center, text width=1.25cm, white] at (6.00,0)
      {\scriptsize Online A/B};
    \node[align=center, text width=1.25cm, white] at (7.40,0)
      {\scriptsize Lineage \&\\[-1pt]\scriptsize deprecation};

    \node[below, font=\scriptsize, black!60, align=center] at (1.25,-1.02)
      {hours};
    \node[below, font=\scriptsize, black!60, align=center] at (3.20,-1.02)
      {hours };
    \node[below, font=\scriptsize, black!60, align=center] at (4.60,-1.02)
      {days};
    \node[below, font=\scriptsize, black!60, align=center] at (6.00,-1.02)
      {weeks};
    \node[below, font=\scriptsize, black!60, align=center] at (7.40,-1.02)
      {weeks};

    \node[right, font=\scriptsize, black!60] at (-0.05,-1.78)
      {Bottleneck:};
    \foreach \x/\c/\l in {1.55/fInf/{inference},
                            3.30/fTrain/{training},
                            4.85/fOnline/{traffic},
                            6.20/fCoord/{coordination}} {
      \fill[\c, draw=black!35, line width=0.3pt]
        (\x,-1.78) rectangle ++(0.18,0.15);
      \node[right, font=\scriptsize, black!60] at ({\x+0.22},-1.70) {\l};
    }

  \end{tikzpicture}

  \caption{FORTE's deprecation funnel. Candidate sets narrow through validation
  and deprecation; fill identifies each stage's bottleneck.}
  \label{fig:funnel}
  \Description{A left-to-right funnel of five narrowing stages: temporal
    feature importance and candidate construction, grouped perturbation,
    retraining, online A/B, and lineage and deprecation. Each stage is
    annotated with what it consumes, and fill colour indicates whether the
    stage is bottlenecked by amortized inference, training capacity, online
    traffic, or organizational coordination.}
\end{figure}

\section{System Design and Methodology}
\label{sec:system}

FORTE assembles deprecation candidates through the deprecation funnel illustrated in
Figure~\ref{fig:funnel}. Its stages differ in which resource they consume, and
online A/B traffic is the bottleneck: it admits only a handful of arms.
The funnel is therefore ordered so that stages bound by amortized inference
screen the full feature set before any candidate reaches a stage bound by
training capacity or online traffic.

\subsection{Temporal feature importance}
\label{subsec:temporal}

We set up a scheduled perturbation experiment for each model. A resilient
workflow reads a partition of the evaluation dataset and evaluates the current
production model.

At industry scale across a diverse model fleet we implement several safeguards
to ensure continuous data collection. Upstream issues can lead to partitions
being delayed or missing; we include partition sensors in our design to skip and
backfill. In a production setting, teams have variable model iteration cycles.
We define the current production model by implementing a production model sensor
to automatically update the model reference as serving traffic diverts to a newer
checkpoint.

For feature $X_j$ on day $t$, temporal feature importance is the signed
normalized change $\mathrm{FI}_j(t)=\Delta_M(\{X_j\};t)$ defined in
Equation~\eqref{eq:normalized-change}. The baseline and perturbed conditions use the same
checkpoint and the same day's examples. This paired comparison controls level
shifts shared by both conditions, while normalization by the baseline expresses
the result as a relative percentage. It does not remove permutation noise or
genuine changes in the model's sensitivity to a feature.

\stagelead{Candidate construction.} We express the temporal gate in the
conventional reliance orientation, $I_j(t)=-\mathrm{FI}_j(t)$, so larger values
indicate greater measured reliance. Across each configured lookback window, we
use mean importance, rank stability, and fitted slope as indicators for
deprecation candidates. By default, a candidate must have low mean importance,
remain stably near the bottom of the feature ranking, and have a non-positive
slope, with these conditions required to hold across every output head and
lookback window. The gate is a per-model configuration rather than a fixed rule:
owning teams set the window lengths, the filter thresholds, and which heads a
candidate must satisfy, according to operational concerns such as model release
frequency and the amount of post-release evidence required before nominating a
feature. The default is deliberately conservative because a candidate that
survives it still has to clear the more expensive gates downstream. Additional
signals include estimated feature costs. We use these signals to generate
multiple nested candidate sets.

Two examples in Figure~\ref{fig:temporal-rank-trajectories} show how
rank changes over time. Feature A remains relatively low-ranked but varies
substantially. Feature B is generally higher-ranked before moving sharply near
the end of the window. This variation motivates evaluating mean rank, stability,
and trend over a lookback window rather than using one observation.

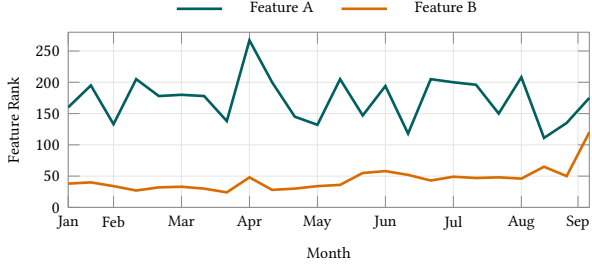
\begin{figure}[t]
  \centering
  \begin{tikzpicture}
    \begin{axis}[
        width=\columnwidth,
        height=0.46\columnwidth,
        xmin=1, xmax=24,
        ymin=0, ymax=280,
        xtick={1,3,6,9,12,15,18,21,23.5},
        xticklabels={Jan,Feb,Mar,Apr,May,Jun,Jul,Aug,Sep},
        ytick={0,50,100,150,200,250},
        xlabel={Month},
        ylabel={Feature Rank},
        xlabel style={font=\scriptsize},
        ylabel style={font=\scriptsize},
        tick label style={font=\scriptsize},
        axis line style={draw=black!45},
        tick style={draw=black!45},
        grid=major,
        grid style={draw=black!10},
        legend style={
          at={(0.5,1.03)},
          anchor=south,
          font=\scriptsize,
          draw=none,
          legend columns=2,
          column sep=8pt,
        },
        legend cell align={left},
      ]
      \addplot[
        color=teal!75!black,
        solid,
        mark=none,
        line width=1pt,
      ] coordinates {
        (1,160) (2,195) (3,133) (4,205)
        (5,178) (6,180) (7,178) (8,138)
        (9,267) (10,200) (11,145) (12,132)
        (13,205) (14,147) (15,194) (16,118)
        (17,205) (18,200) (19,196) (20,150)
        (21,208) (22,111) (23,135) (24,175)
      };
      \addlegendentry{Feature A}

      \addplot[
        color=orange!85!black,
        solid,
        mark=none,
        line width=1pt,
      ] coordinates {
        (1,38) (2,40) (3,34) (4,27)
        (5,32) (6,33) (7,30) (8,24)
        (9,48) (10,28) (11,30) (12,34)
        (13,36) (14,55) (15,58) (16,52)
        (17,43) (18,49) (19,47) (20,48)
        (21,46) (22,65) (23,50) (24,120)
      };
      \addlegendentry{Feature B}
    \end{axis}
  \end{tikzpicture}
  \caption{Feature-rank trajectories for two features across
  scheduled observations. Lower ranks denote more highly ranked features. The
  observations retain their original order; the horizontal labels show the
  corresponding abbreviated month names.}
  \label{fig:temporal-rank-trajectories}
  \Description{A two-line chart showing feature rank from January through
  September. Feature A fluctuates between ranks 111 and 267. Feature B remains
  mostly between ranks 24 and 65 before moving to rank 120 in September. The
  vertical axis begins at 0 and increases upward.}
\end{figure}

We do not assert that this method rigorously identifies all non-correlated
candidates, but we assert that it will provide a signal that is validated in a
future gate.

\subsection{Grouped perturbation}
\label{subsec:grouped}

We take the nested feature sets and perform grouped perturbation on each to
validate their joint perturbation effect, capturing interactions and dependence.
This joint test addresses one operational consequence of feature dependence
\cite{hooker2021,nicodemus2010}: correlated candidates can substitute for one
another under individual permutation, whereas perturbing them together measures
reliance on their combined signal. Grouping does not produce conditional
importance or eliminate out-of-distribution perturbations.
The pooled inference run shares its principal fixed costs across candidate sets.
At this stage we seek to reduce the number of candidate sets carried into
retraining because each surviving set typically takes days to retrain, considerably longer than offline ablation. We first filter candidate
sets for metric neutrality, then retain up to $N$ passing sets ranked by their
potential cost savings from feature deprecation. The value of $N$ and the sizes
of these sets are team-specific choices reflecting cost priorities, retraining
capacity, and available A/B traffic.

\subsection{Retraining, online validation, launch review}
\label{subsec:validation}

We retrain a model with each candidate set removed and compare it with the
production baseline. Neutral performance metrics are the criterion to move to the online
validation stage.

We promote $S$ of the candidate sets that passed retraining to treatment arms.
A/B validation consists of running a test for multiple weeks and benchmarking each arm against the production control. A/B test completion has the following
criteria:

\begin{itemize}
  \item Online metric neutrality.
  \item User behavioral guardrail adherence.
\end{itemize}

Passing the above two allows the team to move this to human audit and launch
review to finally ship the deprecation.

\begin{figure}[!t]
  \centering
  \resizebox{0.85\columnwidth}{!}{%
  \begin{tikzpicture}[x=1cm, y=1cm, font=\small]

    \definecolor{fInf}{HTML}{FDEAD9}
    \definecolor{fPert}{HTML}{D5E3EF}

    \tikzset{
      pipe/.style={draw=black!35, line width=0.4pt, fill=black!5,
                   minimum height=0.46cm, minimum width=#1,
                   font=\scriptsize, align=center, inner sep=1pt},
      flow/.style={-{Stealth[length=1.5mm]}, black!55, line width=0.5pt},
    }

    \node[pipe=2.05cm] (parq)  at (1.05,0) {Parquet dataset};
    \node[pipe=2.05cm] (batch) at (3.60,0) {Batch ($B$ rows)};
    \node[pipe=2.05cm] (coll)  at (6.15,0) {Collate};

    \draw[flow] (parq)  -- (batch);
    \draw[flow] (batch) -- (coll);

    \draw[black!55, line width=0.5pt] (6.15,-0.23) -- (6.15,-0.62);
    \draw[flow] (6.15,-0.62) -- (6.15,-0.95);
    \draw[flow] (6.15,-0.62) -- (2.05,-0.62) -- (2.05,-0.95);

    \newcommand{\fortedict}[7]{%
      \fill[#3, draw=black!35, line width=0.4pt] (#1,-0.95) rectangle (#2,-2.33);
      \draw[black!35, line width=0.4pt] (#1,-1.29) -- (#2,-1.29);
      \node[font=\scriptsize] at ({(#1+#2)/2},-1.12) {#4};
      \node[right, font=\scriptsize] at ({#1+0.22},-1.50) {#5};
      \node[right, font=\scriptsize] at ({#1+0.22},-1.82) {#6};
      \node[right, font=\scriptsize] at ({#1+0.22},-2.14) {#7};
    }

    \fortedict{0.20}{3.90}{fInf}{\textbf{original}}
      {$f_1$: dense}{$f_2$: sparse}{$f_3$: ragged}
    \fortedict{4.30}{8.00}{fPert}{\textbf{perturbed}}
      {$f_1$: zero tensor}{$f_2$: empty pairs}{$f_3$: empty list}

    \draw[flow] (1.20,-2.33) -- (1.20,-2.85);
    \draw[flow] (2.60,-2.33) -- (2.60,-2.85);
    \draw[flow] (6.15,-2.33) -- (6.15,-2.85);
    \node[right, font=\scriptsize, black!60] at (1.26,-2.59) {$f_1$};
    \node[right, font=\scriptsize, black!60] at (2.66,-2.59) {$f_2$};
    \node[right, font=\scriptsize, black!60] at (6.21,-2.59) {$f_3$};

    \fill[white, draw=black!35, line width=0.4pt]
      (0.20,-2.85) rectangle (8.00,-4.23);
    \draw[black!35, line width=0.4pt] (0.20,-3.19) -- (8.00,-3.19);
    \node[font=\scriptsize] at (4.10,-3.02) {experiment (ablate $f_3$)};
    \node[right, font=\scriptsize] at (0.42,-3.40) {$f_1$: dense \textcolor{black!60}{(original)}};
    \node[right, font=\scriptsize] at (0.42,-3.72) {$f_2$: sparse \textcolor{black!60}{(original)}};
    \node[right, font=\scriptsize] at (0.42,-4.04) {$f_3$: empty list \textcolor{black!60}{(perturbed)}};

  \end{tikzpicture}%
  }

  \caption{Dataset virtualization (ablation): the \emph{original} dictionary
    holds typed batches; \emph{perturbed} holds type-appropriate nulls.
    Experiments draw each feature column ($f_n$) from one of these
    dictionaries.}
  \label{fig:dataset-virtualization}
  \Description{Two feature-keyed dictionaries (original and perturbed) feed
    into a third box representing one experiment that ablates feature f3: f1
    and f2 are pulled from original, f3 from perturbed.}
\end{figure}
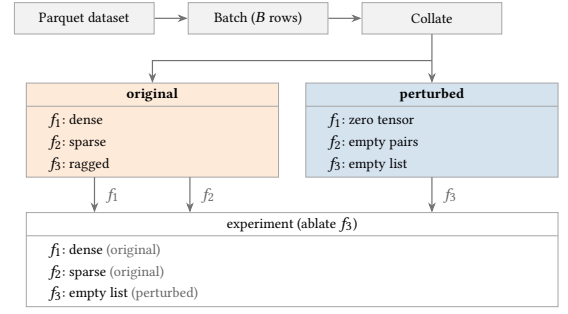

\subsection{Feature lineage and deprecation}
\label{subsec:lineage}

FORTE automatically detects dereferenced features by maintaining a live
feature--model mapping. Once a feature has been removed from the model, FORTE
determines whether it is entirely dereferenced. Teams can then initiate the
formal deprecation process, which progressively blocks the feature and prevents
any model that still references it from training. Blocking and removing
features can have real costs if they need to be restored or revived, so
deprecation proceeds through a staged manual audit process:

\begin{itemize}
  \item \textbf{Dataloader block.} If a feature is attempted to be added back
    into a model it will trigger here.
  \item \textbf{Dataset block.} The feature is hidden from reads, and the block
    propagates to downstream leaf datasets.
  \item \textbf{Ingestion block.} Evaluation and training datasets block a
    blocked feature from being written.
  \item \textbf{Removal from logging.} The final step, which removes the feature
    entirely from logging.
\end{itemize}

\subsection{Dataset Virtualization and Pooled Execution}
\label{subsec:dataset-virtualization}

Individual-feature importance permutes or ablates a feature after reading in and
materializing the dataset, once per experiment. We instead store an
ablated/permuted in-memory copy of the entire batched dataset, so that all
experiments for a model resolve against a single materialization in one
distributed inference pass.

The evaluation dataset is laid out on disk in a columnar
storage format. Rows are read from these files
and collated into batches, converting each raw column into the tensor
format its feature registry schema declares: dense values, sparse
(index, value) pairs, or ragged variable-length lists. Collation yields a
tensor dictionary $D_{\text{original}}$ keyed by feature name. We build
$D_{\text{perturbed}}$ alongside it, mirroring $D_{\text{original}}$ in
keys and batch dimensions, each value the perturbed form of that
feature: a permuted batch under permutation, a type-appropriate null
under ablation.

As Figure~\ref{fig:dataset-virtualization} illustrates, an experiment is a
per-feature selection between the two dictionaries. Each inference worker receives both once per batch and
iterates over the configured experiment set: the baseline draws every
feature from $D_{\text{original}}$, and each subsequent experiment draws
its feature group $G$ from $D_{\text{perturbed}}$ and the remainder from
$D_{\text{original}}$, runs the already-loaded model, and tags the output
with an experiment index. Permutation and ablation share this loop,
differing only in how $D_{\text{perturbed}}$ is constructed. Model
loading, data I/O, actor startup, and GPU allocation amortize across
every experiment in the batch, so $p$ experiments cost one
materialization rather than $p$. Section~\ref{sec:evaluation} measures the
resulting runtime behavior (Figure~\ref{fig:runtime-scaling}).

\section{Practical Considerations for Production Deployment}
\label{sec:practical}

\subsection{Single-Permutation Estimation at Scale}
\label{subsec:single-permutation-proof}

Repeating permutations reduces Monte Carlo noise but multiplies inference cost,
which could instead be spent on inference over a larger dataset. We use a
theory-to-empirical bridge: following tractable ROC-AUC analyses
\cite{agarwal2005auc}, we first establish a
reference variance-scaling law for ROC-AUC, whose pairwise decomposition
supports a clean bounded-differences argument. We then test whether the resulting operating
conclusion extends empirically to PR-AUC and production feature-ranking
decisions. This second step supports an operational generalization; it does not
prove a PR-AUC scaling law.

Fix an evaluation dataset with $n_+$ positive and $n_-$ negative examples,
$n=n_++n_-$, a feature group $G$, and a frozen model. Let $\Pi_n$ denote the
finite set of permutations allowed by the evaluation implementation. For a
global permutation, $\Pi_n=S_n$; because FORTE permutes within inference
batches, its implemented space is the product of the corresponding
within-batch permutation sets. For $\pi\in\Pi_n$, let $A_G(\pi)$ denote the
tie-adjusted AUC after permuting only $G$.

The analysis is conditional on both the fixed evaluation data and the frozen
model architecture and evaluation state. By \emph{row-wise inference}, we mean
that each score has the form $s_i=f(x_i)$: once model state is fixed, changing
the features assigned to row $i$ cannot change another row's score. This
condition excludes inference-time batch normalization that recomputes statistics
from the current batch, joint listwise or slate-level scoring and normalization,
cross-candidate attention, and GNN neighborhood aggregation when perturbing one
row changes neighboring representations or scores. A model trained with a
listwise objective still satisfies the condition if its deployed inference is
row-wise. The Pinterest models evaluated here use frozen evaluation state and
score examples independently.

Each admissible permutation breaks the alignment between $G$ and the remaining
features in a different way. Since no admissible permutation is privileged,
the natural fixed-dataset estimand is their uniform average,
\[
  \mu_{n,G}
  = \mathbb{E}_{\pi\sim\operatorname{Unif}(\Pi_n)}[A_G(\pi)].
\]
This estimand is defined over the permutation space implemented by the system;
it need not equal the average over all of $S_n$ when permutation is restricted
to batches.

For $R$ independent uniform draws with replacement $\pi_1,\ldots,\pi_R$ from
$\Pi_n$, define
\[
  \widehat A_{R,G}=\frac{1}{R}\sum_{r=1}^R A_G(\pi_r).
\]
The standard Monte Carlo identities give
\[
  \mathbb{E}[\widehat A_{R,G}]=\mu_{n,G},
  \qquad
  \operatorname{Var}(\widehat A_{R,G})=\frac{\sigma_{n,G}^2}{R},
  \qquad
  \sigma_{n,G}^2=\operatorname{Var}_{\pi}(A_G(\pi)).
\]
If $A_0>0$ is the fixed, unpermuted AUC, FORTE's normalized importance estimate is
\[
  \widehat\Delta_R(G)
  =100\frac{\widehat A_{R,G}-A_0}{A_0}.
\]
It has expectation $100(\mu_{n,G}-A_0)/A_0$ and variance
$10^4\sigma_{n,G}^2/(R A_0^2)$. A single draw is therefore already unbiased for
$\mu_{n,G}$, and averaging $R$ draws only reduces variance by a factor of $1/R$
without changing the fixed-dataset estimand.

The remaining question is how large $\sigma_{n,G}^2$ itself is, since that is what
sets the precision a single draw achieves.

\begin{proposition}
\label{prop:auc-permutation-variance}
Conditional on the fixed evaluation dataset and frozen model state, assume
$n_+,n_-=\Theta(n)$ and row-wise inference as defined above. Then, for either
global permutation or FORTE's product of within-batch permutation spaces,
\[
  \sigma_{n,G}^2
  =\operatorname{Var}_{\pi\sim\operatorname{Unif}(\Pi_n)}[A_G(\pi)]
  =O\!\left(\frac{1}{n}\right).
\]
\end{proposition}

\begin{proof}
Fix the dataset and frozen model state, and view
$A_G:\Pi_n\to[0,1]$ as a function of the permutation alone. For a transposition
$\tau=(a\,b)$ acting within one permutation block,
$\pi$ and $\pi\circ\tau$ differ only in the values assigned to rows $a$ and
$b$; under row-wise inference only those two scores change. A pairwise
indicator can therefore change only for pairs $(k,l)\in I^+\times I^-$ with
$\{k,l\}\cap\{a,b\}\neq\emptyset$, of which there are at most
$2\max(n_+,n_-)\leq 2n$. Each contributes weight $1/(n_+n_-)$, giving the
per-swap bound
\[
  \bigl|A_G(\pi)-A_G(\pi\circ\tau)\bigr|
  \;\leq\; c_n
  \;=\; \frac{2n}{n_+n_-}
  \;=\; O\!\left(\tfrac{1}{n}\right)
\]
under $n_+,n_-=\Theta(n)$. A bounded-differences inequality for uniform random
permutations \cite{mcdiarmid2002}, applied globally or independently
across the within-batch blocks, then gives $\sigma_{n,G}^2\leq Cnc_n^2=O(1/n)$ for
a universal constant $C$.
\end{proof}

Two levers therefore reduce the variance of a permutation estimate, and they
act on different axes: repetition divides it by $R$ at fixed evaluation size,
while a larger evaluation sample shrinks $\sigma_{n,G}^2$ itself. Repeated
permutations are best understood as variance reduction rather than a change of
estimand, and one whose marginal benefit is small in the large-$n$ regime
characteristic of production evaluation tables.

\stagelead{Empirical bridge to PR-AUC.}
The Monte Carlo identities extend directly to PR-AUC. Let $P_G(\pi)$ denote
PR-AUC after permutation and
$\widehat P_{R,G}=R^{-1}\sum_{r=1}^R P_G(\pi_r)$. Then
$\mathbb{E}[\widehat P_{R,G}]=\mathbb{E}_{\pi}[P_G(\pi)]$ and
$\operatorname{Var}(\widehat P_{R,G})=\operatorname{Var}_{\pi}(P_G(\pi))/R$; the
same statements hold after subtracting and normalizing by the fixed baseline.
However, the $O(1/n)$ rate in Proposition~\ref{prop:auc-permutation-variance}
uses AUC's pairwise-comparison representation and does not establish a
corresponding sample-size rate for PR-AUC, particularly for highly imbalanced
heads. We evaluate that behavior empirically below.

The $1/R$ Monte Carlo identities hold for any fixed architecture because they
concern averages of independent permutation draws. Only the $O(1/n)$ reference
rate depends on row-wise sensitivity; architectures outside that scope require
a separate sensitivity analysis.

We test this rate directly. We evaluated ten permutations for 381 feature
groups using 15k and 150k examples. Across two output heads and two metrics,
increasing the sample size by an order of magnitude reduced the median per-feature permutation standard
deviation by $2.58$--$3.16\times$. The comparison is descriptive
rather than a direct test of
Proposition~\ref{prop:auc-permutation-variance}: the proposition bounds the
variance rather than fixing its constant, the two samples may differ in
composition and class balance, and the proposition does not establish a
corresponding rate for PR-AUC. This comparison forms the empirical bridge:
PR-AUC does not inherit the proposition, but its observed variance reduction
tracks the ROC-AUC reference closely enough that it does not alter the downstream
feature-ranking decisions reported below.

At 150k examples, a single permutation had a median leave-one-out rank
correlation of $0.907$--$0.959$ with the estimate obtained by averaging the
other nine permutations, and median top-20 overlap of $92$--$100\%$. On PR-AUC,
Clickthrough had the weaker standard-deviation reduction, $2.580\times$ compared
with Closeup's $3.082\times$. However, Closeup had the weaker leave-one-out
correlation, $0.907$ compared with Clickthrough's $0.942$. Thus, at the
larger evaluation size, permutation randomness had limited effect on the
feature-ranking decisions for which FORTE is used.

\begin{figure}[t]
  \centering
  \begin{tikzpicture}
    \begin{axis}[
        width=\columnwidth,
        height=0.46\columnwidth,
        xmin=10, xmax=155,
        ymin=0.55, ymax=0.95,
        xtick={15,30,50,80,150},
        xticklabels={{15k\\(10)},{30k\\(5)},
          {50k\\(3)},{80k\\(2)},{150k\\(1)}},
        ytick={0.55,0.65,0.75,0.85,0.95},
        xlabel={Dataset size (permutations)},
        ylabel={Spearman $\rho$ with $1.5\mathrm{M}\!\times\!1$},
        xlabel style={font=\scriptsize},
        ylabel style={font=\scriptsize},
        tick label style={font=\scriptsize},
        xticklabel style={align=center},
        axis line style={draw=black!45},
        tick style={draw=black!45},
        grid=major,
        grid style={draw=black!10},
        legend style={
          at={(0.5,1.03)},
          anchor=south,
          font=\scriptsize,
          draw=none,
          legend columns=2,
          column sep=7pt,
          row sep=1pt,
        },
        legend cell align={left},
      ]

      \addplot[
        color=teal!75!black,
        solid,
        mark=*,
        mark size=1.7pt,
        line width=1pt,
      ] coordinates {
        (15,0.7570)
        (30,0.8465)
        (50,0.8798)
        (80,0.9204)
        (150,0.9377)
      };
      \addlegendentry{AUC Closeup}

      \addplot[
        color=orange!85!black,
        solid,
        mark=square*,
        mark size=1.7pt,
        line width=1pt,
      ] coordinates {
        (15,0.8057)
        (30,0.8216)
        (50,0.8551)
        (80,0.8941)
        (150,0.9194)
      };
      \addlegendentry{AUC Clickthrough}

      \addplot[
        color=blue!65!black,
        solid,
        mark=triangle*,
        mark size=1.9pt,
        line width=1pt,
      ] coordinates {
        (15,0.5762)
        (30,0.7346)
        (50,0.7909)
        (80,0.8266)
        (150,0.8564)
      };
      \addlegendentry{PR-AUC Closeup}

      \addplot[
        color=red!70!black,
        solid,
        mark=diamond*,
        mark size=1.9pt,
        line width=1pt,
      ] coordinates {
        (15,0.6556)
        (30,0.7593)
        (50,0.8004)
        (80,0.8311)
        (150,0.8818)
      };
      \addlegendentry{PR-AUC Clickthrough}

    \end{axis}
  \end{tikzpicture}

  \caption{AUC and PR-AUC feature-ranking correlations across approximately
  matched row-evaluation budgets, measured against a separate
  $1.5\mathrm{M}\!\times\!1$ reference. Tick labels show dataset size followed by the number of
  permutations in parentheses.}
  \label{fig:single-permutation}
  \Description{A line plot of AUC and PR-AUC feature-ranking correlations for
  the Closeup and Clickthrough heads against dataset size. Each x-axis tick also gives
  the corresponding permutation count. All four correlations generally increase
  as the dataset grows and the number of permutations decreases.}
\end{figure}
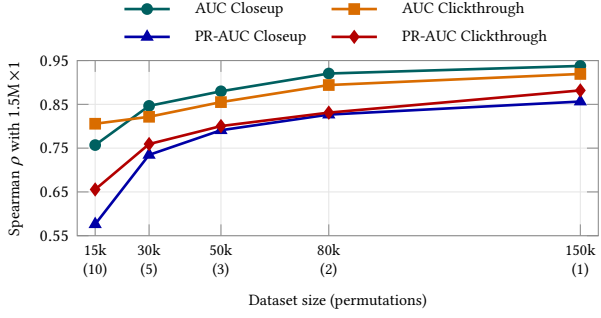

Figure~\ref{fig:single-permutation} shows that feature-ranking agreement with the
larger reference rises monotonically as the row-to-permutation ratio increases.
All four metric--head combinations agree more at $150\mathrm{k}\!\times\!1$ than
at $15\mathrm{k}\!\times\!10$, which processes the same number of
row-evaluations, by between $0.11$ (AUC Clickthrough) and $0.28$ (PR-AUC Closeup) in
Spearman $\rho$, with no reversal at any intermediate configuration.

Across both metrics and heads, the evidence supports $R=1$ at the evaluated
production-scale sample size of roughly 150k examples when the goal is feature
ranking, and allocating compute to more representative rows when possible. Our
measurements bear on rank correlation and top-20 overlap, not on the accuracy of
any individual feature's importance value. The 15k and
150k samples are not confirmed to be nested, only ten repeats are available, and
the 1.5M reference uses one permutation, so it is a larger-sample comparison
rather than ground truth. Three to five repeats remain useful for smaller samples
or noisy tail metrics when more representative rows are unavailable.

\subsection{Engineering for Production-Scale Evaluation}
\label{subsec:production-scale-evaluation}

\subsubsection{Feature storage and type-awareness during perturbation}

Standard libraries define permutation as a row-level operation. FORTE instead
applies it to the batched representation the model actually consumes, which
requires a custom permutation kernel per feature type. Datasets are defined at a
row level and batched at the dataloader--model boundary, so operating after that
boundary is what lets a single materialization serve every experiment
(Section~\ref{subsec:dataset-virtualization}).

Ablation raises the same requirement in a different form: each of those feature
types needs a null or empty representation. Where feature producers and
consumers span an organization, these definitions cannot live in the evaluation
job alone. We formalize them in a Feature Registry that categorizes every
supported feature across the ML stack, so a type-appropriate null resolves
identically in training and in evaluation.

\subsubsection{Production model sensor}

Models are iterated on regularly. To support
automatic daily feature importance, FORTE resolves the model from serving
traffic, ensuring we can track importance on the latest model across our model
fleet.

\subsubsection{Randomization}
\label{subsec:randomization}

Production evaluation tables are generally ordered by ingestion time, partition
key, or some upstream join order. They do not guarantee random ordering.
Randomizing row order before taking a subset reduces dependence on this physical
layout. A full distributed shuffle does so at the cost of communicating data
across workers \cite{attia2016shuffling}. We define a two-stage shuffle, related
to block-aware hierarchical shuffling \cite{xu2022corgipile}, to reduce layout
correlation: a Spark job randomizes rows within each output file, and the
inference layer dispatches files to workers in randomized order. Together, the
arrival order of inference batches is decoupled from the source table's physical
layout.

For international data we take care not to oversample any single market; the
fairness implications are discussed in the ethical considerations section.

\raggedbottom
\AddToHookNext{shipout/after}{\flushbottom}

\section{Evaluation}
\label{sec:evaluation}

Unless otherwise noted, the following evaluations use the same frozen production Pin Ranker checkpoint (Section~\ref{subsec:pinranker}) and sample rows from its associated evaluation dataset using the two-stage randomization procedure in Section~\ref{subsec:randomization}. Each subsection reports its evaluation-set size and experiment-specific configuration.

\subsection{Scaling behavior}

We compare a legacy (one-pass-per-feature) implementation of feature perturbation experiments against
our pooled execution implementation, isolating perturbation experiments from
dataset sampling.  We vary two axes, dataset size and number of perturbation
experiments, choosing a random subset of features for each experiment.
Evaluation-set sizes are 30k (small), 300k (medium), and 3M (large) rows, and
all runs use the same compute allocation.

\begin{figure}[t]
  \centering
  \begin{tikzpicture}
    \begin{loglogaxis}[
        width=0.94\columnwidth,
        height=0.52\columnwidth,
        xlabel={Experiments per run (log scale)},
        ylabel={Wall-clock time (minutes; log scale)},
        xlabel style={font=\scriptsize},
        ylabel style={font=\scriptsize},
        tick label style={font=\scriptsize},
        xmin=0.8, xmax=1500,
        ymin=3, ymax=720,
        xtick={1,10,100,1000},
        xticklabels={1,10,100,1000},
        ytick={5,15,60,360},
        yticklabels={5,15,60,360},
        grid=major,
        grid style={draw=black!10},
        axis line style={draw=black!45},
        tick style={draw=black!45},
        legend style={
          at={(0.5,1.03)},
          anchor=south,
          legend columns=3,
          font=\scriptsize,
          draw=none,
          fill=none,
          column sep=5pt,
          inner sep=0pt
        },
        legend cell align={left},
        clip mode=individual
      ]

      \addplot[
        color=blue!65!black,
        dashed,
        line width=0.6pt,
        opacity=0.55,
        forget plot,
        domain=1:150,
        samples=2
      ] {4.266*x};

      \addplot[
        color=orange!85!black,
        dashed,
        line width=0.6pt,
        opacity=0.55,
        forget plot,
        domain=1:30,
        samples=2
      ] {23.784*x};

      \addplot[
        color=teal!75!black,
        dashed,
        line width=0.6pt,
        opacity=0.55,
        forget plot,
        domain=1:16,
        samples=2
      ] {44.982*x};

      \addplot[
        color=blue!65!black,
        mark=*,
        mark size=1.5pt,
        line width=0.9pt
      ] coordinates {
        (1,4.266)
        (10,4.386)
        (100,5.382)
        (500,10.500)
        (1000,19.134)
      };
      \addlegendentry{30k: $N^{0.20}$}

      \addplot[
        color=orange!85!black,
        mark=square*,
        mark size=1.4pt,
        line width=0.9pt
      ] coordinates {
        (1,23.784)
        (10,24.450)
        (100,33.132)
        (500,72.216)
        (1000,123.186)
      };
      \addlegendentry{300k: $N^{0.23}$}

      \addplot[
        color=teal!75!black,
        mark=triangle*,
        mark size=1.8pt,
        line width=0.9pt
      ] coordinates {
        (1,44.982)
        (100,109.734)
        (200,178.614)
        (300,249.984)
      };
      \addlegendentry{3M: $N^{0.27}$}

    \end{loglogaxis}
  \end{tikzpicture}

  \caption{Runtime grows sublinearly with experiment count for all three
  evaluation-set sizes. Dashed lines show linear scaling from each series'
  single-experiment runtime; the 3M-row set was measured through 300
  experiments.}
  \label{fig:runtime-scaling}
  \Description{A log-log plot of wall-clock time against the number of
    perturbation experiments per run, for evaluation sets of 30k, 300k, and 3M
    rows. All three series rise far more slowly than the dashed linear-scaling
    references anchored at each series' single-experiment runtime.}
\end{figure}
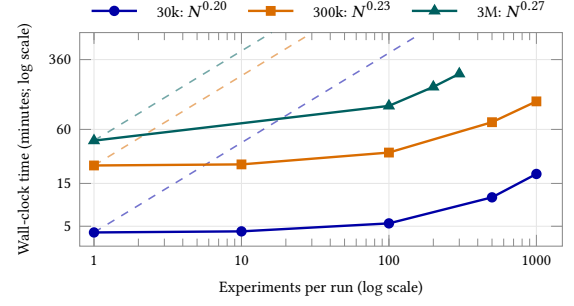

For our largest dataset (3M rows), Figure~\ref{fig:runtime-scaling} shows a
$54\times$ reduction in elapsed time over 300 experiments. The empirical scaling trend shows clear amortization:
runtime grows as $N^{0.20}$, $N^{0.23}$, and $N^{0.27}$ in experiment count for
the small, medium, and large evaluation sets respectively, against the $N^{1}$
scaling of one evaluation pass per feature. For the 30k-row set, a
$1000\times$ increase in experiments costs only $4.5\times$ the wall-clock time.

Adding an additional experiment costs 41\,s at 3M rows, against 45 minutes for
a separate pass under the one-pass-per-feature approach. The exponent rises with
evaluation-set size because forward-pass cost takes a larger share of the total
as rows grow: model loading, data I/O, and actor startup are paid once per run,
while forward passes and final writes grow with the experiment count. It
remains far below linear throughout.

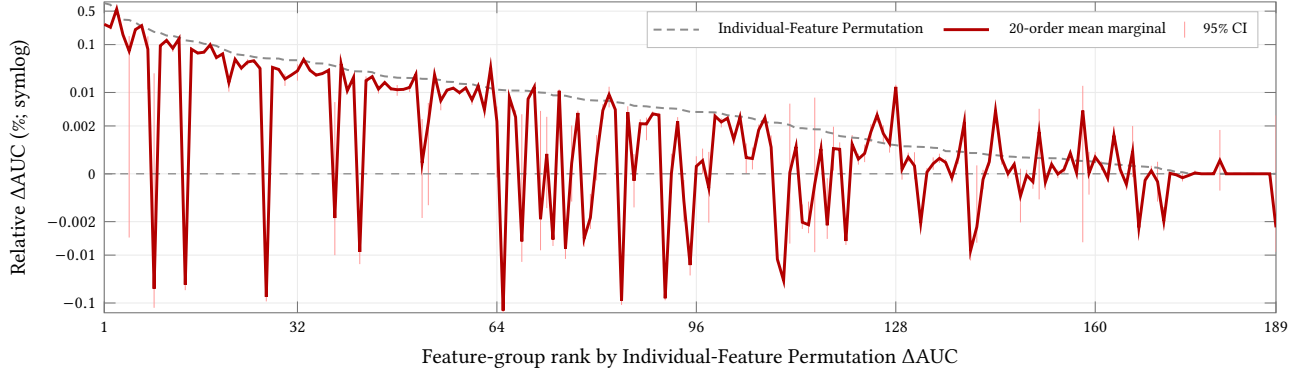
\begin{figure*}[!t]
  \centering
  \pgfplotstableread[col sep=comma]{
rank,individual,shapley
1,0.7351925554181175,0.265512905047815
2,0.6800957153131881,0.23040110816654596
3,0.4410925627218412,0.5517175613006062
4,0.3317078215272971,0.1615030707057486
5,0.3261642382771162,0.07326405083788765
6,0.2880215697383941,0.20578765304867117
7,0.23968955722617458,0.2487229144379357
8,0.2144766067564936,0.08019184660591944
9,0.18401183664886106,-0.050436555582061504
10,0.17540764683909557,0.09473902640387755
11,0.15445529481788997,0.12255972245392568
12,0.15174761139271706,0.08369621519397787
13,0.14101342119981938,0.1325588685649546
14,0.12305426207263935,-0.04159660639996702
15,0.11788691596584068,0.08007244874311266
16,0.11408220751633676,0.06655582458343112
17,0.1038373297823597,0.06933150869185892
18,0.10112976518860849,0.09918728975579008
19,0.0940309161399033,0.05356639173287599
20,0.08075997740232455,0.06405406242276342
21,0.07159791310761822,0.016100576615603997
22,0.06306380457335459,0.04919164038575692
23,0.060874003957741056,0.0322903490150012
24,0.05903105843884583,0.04320752978814209
25,0.05519716565216492,0.045974363109422864
26,0.05279385768086147,0.03170087914193355
27,0.05161752219393996,-0.07453109947664761
28,0.05120035065618695,0.03342491404286511
29,0.05108981427776531,0.030269306055519595
30,0.04725982500599342,0.019166848711720332
31,0.04722634041285371,0.023057737106445847
32,0.04717169898935801,0.028227376972302782
33,0.046561077332189275,0.04883003731525278
34,0.04286344167201244,0.029151009765277243
35,0.04048208724470977,0.023049441818242766
36,0.040290348208850676,0.02481490369067766
37,0.038357869057768244,0.029044404765982745
38,0.03567261508964154,-0.0018469331442090315
39,0.03283963067246543,0.039548744216404655
40,0.02966004821042973,0.010866032847847332
41,0.027766677357806968,0.02065415092750158
42,0.02753664656679093,-0.008577366977655019
43,0.027445870860394093,0.017687989869124822
44,0.02649027352230068,0.021450106570601736
45,0.02319880773051505,0.011771206265702753
46,0.021707131842105282,0.016146443479220358
47,0.021202494102909416,0.011934251071112054
48,0.020340602714651927,0.011424816252441895
49,0.02017798235758149,0.011604203062418883
50,0.020076672962997484,0.012735401385949325
51,0.019954061467034843,0.02388874647692495
52,0.01928011767429987,0.00045971939271227425
53,0.019193551888522866,0.0023651892252157304
54,0.01897502614761132,0.023445944522037805
55,0.0181346404457264,0.006756151179930811
56,0.016707250269873658,0.011198331371048751
57,0.016572454813296954,0.012327751747614777
58,0.015222281376519537,0.00990543199953271
59,0.014992318153368385,0.012453961614814959
60,0.014858154897285678,0.007022447391629658
61,0.013692659944923743,0.013637146025242065
62,0.013129131349050568,0.004351703164812653
63,0.012235993333923736,0.03280502984518707
64,0.011479811772941224,0.0024735277016200543
65,0.011198335226126583,-0.14425283161038022
66,0.010982818753701615,0.007939827218657942
67,0.010965014025378497,0.0030676617868932088
68,0.010942464121208445,-0.005174063999993458
69,0.010810215386789099,0.007291900313039193
70,0.010297875201176468,0.012999216046146187
71,0.009950569015349554,-0.0018973148681629332
72,0.009792386840904478,0.0008310340233387975
73,0.008699724793142476,-0.004734172908455851
74,0.008679711182176596,0.010858087007440087
75,0.008542077433281912,-0.007291815620915748
76,0.007958705615062652,0.0004438675258969793
77,0.007503615267498973,0.00373694612627678
78,0.0070744708453540275,-0.0044850974972205524
79,0.00687460206960093,-0.0018347976663662137
80,0.006781171826184073,0.0006700847393025899
81,0.006777023010449365,0.004221762603146348
82,0.006640925402950288,0.009020306644891294
83,0.006293833377995715,0.004384984220928754
84,0.006069215069266372,-0.09005853086555794
85,0.005760260274647553,0.003889250172290709
86,0.005262447627435981,-0.0002924774124361429
87,0.005208346747087572,0.0022282818438217103
88,0.00512432063290133,0.0021947584216355748
89,0.004924001061699956,0.0035826712053103394
90,0.0048038051180622505,0.0033688258226486217
91,0.004663993332703739,-0.07975821552633025
92,0.004636278156644297,1.7137181831622383e-05
93,0.004556627063739459,0.002519870028396451
94,0.004228709338970038,-0.0011335517744046266
95,0.0039413602605838545,-0.016131617507560408
96,0.00391612453251522,0.000313422426153606
97,0.0039088188271711585,0.0005506692650838613
98,0.003865930398613263,-0.0005884324071291143
99,0.0038365016419147856,0.0032086311356119613
100,0.0036338587633029927,0.0024186297284713096
101,0.0033405200846840662,0.0029033725219055517
102,0.0032376933501710036,0.0014385040752955714
103,0.0031341743384892227,0.0030286049801642226
104,0.0030384177029927142,0.0006756304808467001
105,0.002857587650938319,0.000637572007950428
106,0.0025929607748565657,0.0018284295273546885
107,0.0023845709529061307,0.002934682663800118
108,0.0023520348073946123,0.0011122708381112317
109,0.002336621248197261,-0.012156159045778294
110,0.0022561403921797073,-0.03385113666322863
111,0.002086844180044984,5.393990516656122e-05
112,0.001907979913115678,0.001135738839480585
113,0.0019058914775587367,-0.0020570201068593587
114,0.0018809152983551793,-0.002322204089035962
115,0.0018586861948415446,-0.0004258447940742371
116,0.001854442196524373,0.0010352124926505085
117,0.0017362690300172484,-0.0023808913981141987
118,0.0016820761985857657,0.0010555088901901418
119,0.0015763014473584554,0.0010241363354087068
120,0.0015452344629454699,-0.005071039589696935
121,0.0015126116542520539,0.0012382607612479123
122,0.0014930238679294375,0.0006148490771176164
123,0.0014537814618172866,0.0008507926314556935
124,0.001407561495301442,0.0017468874982828285
125,0.0013293127300282191,0.0033631910303304592
126,0.0012802121677948307,0.0016810214359802094
127,0.0012257210094295421,0.0012667851287464823
128,0.0011968408746363362,0.01309380032425704
129,0.0011481398443324532,0.0001868378522957296
130,0.00114771798562458,0.0006938279698009642
131,0.001132603194060664,0.00033204602635626793
132,0.0011245097351529907,-0.002116294505968783
133,0.0011141192615106276,6.674293815108062e-05
134,0.001086067567940408,0.0004239394295150717
135,0.0010728778738612215,0.0006444717325916116
136,0.0009475652777882219,0.00047395564867931255
137,0.0009404215884745277,-0.00016519392302389717
138,0.0009315237778256811,0.0006876298952748423
139,0.0009092826295806471,0.002394098446413518
140,0.0008680103956762487,-0.008044157016005532
141,0.0008672593638366213,-0.002580594282733534
142,0.0008311618026594677,-0.00023459628292358953
143,0.0007913252397050649,0.0005006396616125334
144,0.000777249084206013,0.004166759572412751
145,0.0007505148177128989,0.0006370190593736044
146,0.000722378076719984,5.430276960870229e-06
147,0.0007098242553147961,0.0003970138885837922
148,0.0006841746004470978,-0.0009289034651893347
149,0.0006791478427311198,-6.231574193742983e-05
150,0.0006653679577831002,-0.0003325464588952498
151,0.0006573906862514567,0.0017513530793797471
152,0.0006566508177819986,-0.0003521715135860748
153,0.0006509771708710237,0.0003959811274850997
154,0.0006483639090591801,-1.1195381846182407e-05
155,0.000620644913913519,0.0001814027418334696
156,0.0005682189829845887,0.0008664764118372498
157,0.0005474962028599994,4.0860613586182756e-05
158,0.0005349672053155822,0.004192284300762026
159,0.0005305992345817568,2.5053321716803773e-05
160,0.000491209941725784,0.0006944955199119085
161,0.0004481634630538973,0.0003425306838196765
162,0.00041896781556249483,-0.00018283717783132007
163,0.0003970058990293314,0.001604483412299934
164,0.0003746390157212419,0.00018811995306334633
165,0.0003531597689496335,-0.0004421470702474825
166,0.00032117797396320785,0.0009296611624420387
167,0.0003134423311342785,-0.002654035692258201
168,0.00029765612028899506,-0.00027282110450789585
169,0.00027558271669414914,0.00012670027881672222
170,0.00025513975581121666,-0.0003274735846513846
171,0.0001621498788324109,-0.0019784183045674637
172,0.0001463065290754208,1.959122348355047e-05
173,7.608468303636812e-05,-4.306550064462899e-05
174,6.626574568348889e-05,-0.00016430895696211006
175,1.001650000962931e-05,-6.91918909580836e-05
176,2.147190227849826e-06,3.68922979771702e-05
177,0.0,0.0
178,0.0,0.0
179,0.0,0.0
180,0.0,0.00056151710693646
181,0.0,0.0
182,0.0,0.0
183,0.0,0.0
184,0.0,0.0
185,0.0,0.0
186,0.0,0.0
187,0.0,0.0
188,0.0,0.0
189,0.0,-0.0026295001623369172
  }\individualshapleydata
  \def\shapleycibars{%
    \draw[red!40, line width=0.35pt]
      (axis cs:1,0.25195735055753743) -- (axis cs:1,0.27906845953809256);
    \draw[red!40, line width=0.35pt]
      (axis cs:2,0.20769115856155054) -- (axis cs:2,0.25311105777154136);
    \draw[red!40, line width=0.35pt]
      (axis cs:3,0.5295789595176313) -- (axis cs:3,0.5738561630835809);
    \draw[red!40, line width=0.35pt]
      (axis cs:4,0.14937436106756366) -- (axis cs:4,0.17363178034393348);
    \draw[red!40, line width=0.35pt]
      (axis cs:5,-0.004294018625222051) -- (axis cs:5,0.15082212030099734);
    \draw[red!40, line width=0.35pt]
      (axis cs:6,0.197954936512119) -- (axis cs:6,0.21362036958522332);
    \draw[red!40, line width=0.35pt]
      (axis cs:7,0.23811359270533977) -- (axis cs:7,0.25933223617053175);
    \draw[red!40, line width=0.35pt]
      (axis cs:8,0.011518187915362105) -- (axis cs:8,0.14886550529647677);
    \draw[red!40, line width=0.35pt]
      (axis cs:9,-0.12600528235144343) -- (axis cs:9,0.025132171187320404);
    \draw[red!40, line width=0.35pt]
      (axis cs:10,0.08908529035816723) -- (axis cs:10,0.10039276244958786);
    \draw[red!40, line width=0.35pt]
      (axis cs:11,0.1187579700143865) -- (axis cs:11,0.1263614748934649);
    \draw[red!40, line width=0.35pt]
      (axis cs:12,0.0700931929831047) -- (axis cs:12,0.09729923740485104);
    \draw[red!40, line width=0.35pt]
      (axis cs:13,0.12823695268769475) -- (axis cs:13,0.1368807844422144);
    \draw[red!40, line width=0.35pt]
      (axis cs:14,-0.053939260651647086) -- (axis cs:14,-0.029253952148286952);
    \draw[red!40, line width=0.35pt]
      (axis cs:15,0.07185574016795468) -- (axis cs:15,0.08828915731827064);
    \draw[red!40, line width=0.35pt]
      (axis cs:16,0.06223779433905022) -- (axis cs:16,0.07087385482781203);
    \draw[red!40, line width=0.35pt]
      (axis cs:17,0.06564522353560887) -- (axis cs:17,0.07301779384810897);
    \draw[red!40, line width=0.35pt]
      (axis cs:18,0.08766271688489874) -- (axis cs:18,0.11071186262668145);
    \draw[red!40, line width=0.35pt]
      (axis cs:19,0.04928873797936812) -- (axis cs:19,0.057844045486383854);
    \draw[red!40, line width=0.35pt]
      (axis cs:20,0.06163577250529298) -- (axis cs:20,0.06647235234023385);
    \draw[red!40, line width=0.35pt]
      (axis cs:21,0.010392615665401169) -- (axis cs:21,0.021808537565806824);
    \draw[red!40, line width=0.35pt]
      (axis cs:22,0.04570166242016352) -- (axis cs:22,0.05268161835135032);
    \draw[red!40, line width=0.35pt]
      (axis cs:23,0.028867930082797963) -- (axis cs:23,0.03571276794720443);
    \draw[red!40, line width=0.35pt]
      (axis cs:24,0.03701217594065928) -- (axis cs:24,0.049402883635624895);
    \draw[red!40, line width=0.35pt]
      (axis cs:25,0.0447209325564861) -- (axis cs:25,0.04722779366235963);
    \draw[red!40, line width=0.35pt]
      (axis cs:26,0.027941710197496913) -- (axis cs:26,0.03546004808637019);
    \draw[red!40, line width=0.35pt]
      (axis cs:27,-0.0926302024982404) -- (axis cs:27,-0.056431996455054825);
    \draw[red!40, line width=0.35pt]
      (axis cs:28,0.032640832479686514) -- (axis cs:28,0.0342089956060437);
    \draw[red!40, line width=0.35pt]
      (axis cs:29,0.02870407020315681) -- (axis cs:29,0.03183454190788238);
    \draw[red!40, line width=0.35pt]
      (axis cs:30,0.013701494656909724) -- (axis cs:30,0.024632202766530942);
    \draw[red!40, line width=0.35pt]
      (axis cs:31,0.021682240522757846) -- (axis cs:31,0.024433233690133848);
    \draw[red!40, line width=0.35pt]
      (axis cs:32,0.017680659649853343) -- (axis cs:32,0.03877409429475222);
    \draw[red!40, line width=0.35pt]
      (axis cs:33,0.045437392245721166) -- (axis cs:33,0.0522226823847844);
    \draw[red!40, line width=0.35pt]
      (axis cs:34,0.026646548808853955) -- (axis cs:34,0.03165547072170053);
    \draw[red!40, line width=0.35pt]
      (axis cs:35,0.021172710881772302) -- (axis cs:35,0.02492617275471323);
    \draw[red!40, line width=0.35pt]
      (axis cs:36,0.0230467553381794) -- (axis cs:36,0.02658305204317592);
    \draw[red!40, line width=0.35pt]
      (axis cs:37,0.025718820746406804) -- (axis cs:37,0.03236998878555869);
    \draw[red!40, line width=0.35pt]
      (axis cs:38,-0.010008795145713146) -- (axis cs:38,0.006314928857295082);
    \draw[red!40, line width=0.35pt]
      (axis cs:39,0.025445178371847142) -- (axis cs:39,0.05365231006096217);
    \draw[red!40, line width=0.35pt]
      (axis cs:40,0.009965061872557327) -- (axis cs:40,0.011767003823137336);
    \draw[red!40, line width=0.35pt]
      (axis cs:41,0.019526006065759986) -- (axis cs:41,0.02178229578924317);
    \draw[red!40, line width=0.35pt]
      (axis cs:42,-0.015407160935584547) -- (axis cs:42,-0.00174757301972549);
    \draw[red!40, line width=0.35pt]
      (axis cs:43,0.013740695404865014) -- (axis cs:43,0.02163528433338463);
    \draw[red!40, line width=0.35pt]
      (axis cs:44,0.020708101264259742) -- (axis cs:44,0.02219211187694373);
    \draw[red!40, line width=0.35pt]
      (axis cs:45,0.009508472273861125) -- (axis cs:45,0.01403394025754438);
    \draw[red!40, line width=0.35pt]
      (axis cs:46,0.015054805718344331) -- (axis cs:46,0.017238081240096383);
    \draw[red!40, line width=0.35pt]
      (axis cs:47,0.01080194076148638) -- (axis cs:47,0.013066561380737727);
    \draw[red!40, line width=0.35pt]
      (axis cs:48,0.008529914235274116) -- (axis cs:48,0.014319718269609674);
    \draw[red!40, line width=0.35pt]
      (axis cs:49,0.010871508737726708) -- (axis cs:49,0.012336897387111058);
    \draw[red!40, line width=0.35pt]
      (axis cs:50,0.010255241244043733) -- (axis cs:50,0.015215561527854917);
    \draw[red!40, line width=0.35pt]
      (axis cs:51,0.020812058081588782) -- (axis cs:51,0.026965434872261117);
    \draw[red!40, line width=0.35pt]
      (axis cs:52,-0.0018525169606347063) -- (axis cs:52,0.0027719557460592547);
    \draw[red!40, line width=0.35pt]
      (axis cs:53,-0.001314505381125067) -- (axis cs:53,0.006044883831556528);
    \draw[red!40, line width=0.35pt]
      (axis cs:54,0.02155985372218964) -- (axis cs:54,0.025332035321885972);
    \draw[red!40, line width=0.35pt]
      (axis cs:55,0.00424563786375933) -- (axis cs:55,0.009266664496102292);
    \draw[red!40, line width=0.35pt]
      (axis cs:56,0.010454774878457286) -- (axis cs:56,0.011941887863640217);
    \draw[red!40, line width=0.35pt]
      (axis cs:57,0.011548986926529165) -- (axis cs:57,0.01310651656870039);
    \draw[red!40, line width=0.35pt]
      (axis cs:58,0.008913693553409457) -- (axis cs:58,0.010897170445655961);
    \draw[red!40, line width=0.35pt]
      (axis cs:59,0.01158116105784617) -- (axis cs:59,0.013326762171783748);
    \draw[red!40, line width=0.35pt]
      (axis cs:60,0.005525794566766556) -- (axis cs:60,0.00851910021649276);
    \draw[red!40, line width=0.35pt]
      (axis cs:61,0.013131019110289367) -- (axis cs:61,0.014143272940194762);
    \draw[red!40, line width=0.35pt]
      (axis cs:62,0.002902335595304745) -- (axis cs:62,0.005801070734320561);
    \draw[red!40, line width=0.35pt]
      (axis cs:63,0.03013926634541537) -- (axis cs:63,0.03547079334495878);
    \draw[red!40, line width=0.35pt]
      (axis cs:64,-0.00014593341390064423) -- (axis cs:64,0.005092988817140752);
    \draw[red!40, line width=0.35pt]
      (axis cs:65,-0.1576009192444456) -- (axis cs:65,-0.13090474397631482);
    \draw[red!40, line width=0.35pt]
      (axis cs:66,0.0074101329824832396) -- (axis cs:66,0.008469521454832644);
    \draw[red!40, line width=0.35pt]
      (axis cs:67,0.0018970633400648004) -- (axis cs:67,0.004238260233721617);
    \draw[red!40, line width=0.35pt]
      (axis cs:68,-0.013851197393913536) -- (axis cs:68,0.003503069393926619);
    \draw[red!40, line width=0.35pt]
      (axis cs:69,0.006466167430857166) -- (axis cs:69,0.00811763319522122);
    \draw[red!40, line width=0.35pt]
      (axis cs:70,0.011294292004536091) -- (axis cs:70,0.014704140087756282);
    \draw[red!40, line width=0.35pt]
      (axis cs:71,-0.007919409137584807) -- (axis cs:71,0.004124779401258941);
    \draw[red!40, line width=0.35pt]
      (axis cs:72,-0.001431814468997151) -- (axis cs:72,0.003093882515674746);
    \draw[red!40, line width=0.35pt]
      (axis cs:73,-0.0064381806458256715) -- (axis cs:73,-0.0030301651710860312);
    \draw[red!40, line width=0.35pt]
      (axis cs:74,0.010139140442851602) -- (axis cs:74,0.011577033572028573);
    \draw[red!40, line width=0.35pt]
      (axis cs:75,-0.012002220321099205) -- (axis cs:75,-0.0025814109207322935);
    \draw[red!40, line width=0.35pt]
      (axis cs:76,-0.0011533202049946996) -- (axis cs:76,0.002041055256788658);
    \draw[red!40, line width=0.35pt]
      (axis cs:77,0.00328565694814739) -- (axis cs:77,0.00418823530440617);
    \draw[red!40, line width=0.35pt]
      (axis cs:78,-0.0066745736187196255) -- (axis cs:78,-0.002295621375721479);
    \draw[red!40, line width=0.35pt]
      (axis cs:79,-0.002830805331274054) -- (axis cs:79,-0.0008387900014583734);
    \draw[red!40, line width=0.35pt]
      (axis cs:80,-0.0002700613256302925) -- (axis cs:80,0.0016102308042354723);
    \draw[red!40, line width=0.35pt]
      (axis cs:81,0.0027592412866346483) -- (axis cs:81,0.0056842839196580485);
    \draw[red!40, line width=0.35pt]
      (axis cs:82,0.00487761583822726) -- (axis cs:82,0.013162997451555327);
    \draw[red!40, line width=0.35pt]
      (axis cs:83,0.004035782179422513) -- (axis cs:83,0.004734186262434995);
    \draw[red!40, line width=0.35pt]
      (axis cs:84,-0.10938396619458635) -- (axis cs:84,-0.07073309553652954);
    \draw[red!40, line width=0.35pt]
      (axis cs:85,0.0026400399873039654) -- (axis cs:85,0.0051384603572774525);
    \draw[red!40, line width=0.35pt]
      (axis cs:86,-0.0014030045266859158) -- (axis cs:86,0.0008180497018136299);
    \draw[red!40, line width=0.35pt]
      (axis cs:87,0.0016928982045191167) -- (axis cs:87,0.002763665483124304);
    \draw[red!40, line width=0.35pt]
      (axis cs:88,0.0013990924014796752) -- (axis cs:88,0.0029904244417914743);
    \draw[red!40, line width=0.35pt]
      (axis cs:89,0.0031598001907813676) -- (axis cs:89,0.004005542219839311);
    \draw[red!40, line width=0.35pt]
      (axis cs:90,0.002987501623183448) -- (axis cs:90,0.0037501500221137954);
    \draw[red!40, line width=0.35pt]
      (axis cs:91,-0.08763787418058058) -- (axis cs:91,-0.07187855687207992);
    \draw[red!40, line width=0.35pt]
      (axis cs:92,-0.0008078416051499644) -- (axis cs:92,0.0008421159688132093);
    \draw[red!40, line width=0.35pt]
      (axis cs:93,0.0019931294459620925) -- (axis cs:93,0.0030466106108308095);
    \draw[red!40, line width=0.35pt]
      (axis cs:94,-0.0020113630246666714) -- (axis cs:94,-0.0002557405241425816);
    \draw[red!40, line width=0.35pt]
      (axis cs:95,-0.026579061231104664) -- (axis cs:95,-0.005684173784016153);
    \draw[red!40, line width=0.35pt]
      (axis cs:96,-0.00010131439297784848) -- (axis cs:96,0.0007281592452850604);
    \draw[red!40, line width=0.35pt]
      (axis cs:97,4.727940722482537e-05) -- (axis cs:97,0.0010540591229428972);
    \draw[red!40, line width=0.35pt]
      (axis cs:98,-0.0020896721598512033) -- (axis cs:98,0.0009128073455929744);
    \draw[red!40, line width=0.35pt]
      (axis cs:99,0.002810767770743148) -- (axis cs:99,0.0036064945004807745);
    \draw[red!40, line width=0.35pt]
      (axis cs:100,0.0019373624088951152) -- (axis cs:100,0.0028998970480475037);
    \draw[red!40, line width=0.35pt]
      (axis cs:101,0.0025049817523671043) -- (axis cs:101,0.0033017632914439992);
    \draw[red!40, line width=0.35pt]
      (axis cs:102,0.0011782499047556317) -- (axis cs:102,0.001698758245835511);
    \draw[red!40, line width=0.35pt]
      (axis cs:103,0.0026528352990479575) -- (axis cs:103,0.0034043746612804878);
    \draw[red!40, line width=0.35pt]
      (axis cs:104,8.6045268684875e-05) -- (axis cs:104,0.0012652156930085252);
    \draw[red!40, line width=0.35pt]
      (axis cs:105,0.00017792458775944743) -- (axis cs:105,0.0010972194281414086);
    \draw[red!40, line width=0.35pt]
      (axis cs:106,0.0015226679191395276) -- (axis cs:106,0.0021341911355698493);
    \draw[red!40, line width=0.35pt]
      (axis cs:107,0.002264532736961205) -- (axis cs:107,0.003604832590639031);
    \draw[red!40, line width=0.35pt]
      (axis cs:108,0.0006158867562937898) -- (axis cs:108,0.0016086549199286738);
    \draw[red!40, line width=0.35pt]
      (axis cs:109,-0.014334735498218853) -- (axis cs:109,-0.009977582593337735);
    \draw[red!40, line width=0.35pt]
      (axis cs:110,-0.039269905007911256) -- (axis cs:110,-0.028432368318546);
    \draw[red!40, line width=0.35pt]
      (axis cs:111,-0.005739253207027246) -- (axis cs:111,0.0058471330173603675);
    \draw[red!40, line width=0.35pt]
      (axis cs:112,0.0009636011418179067) -- (axis cs:112,0.0013078765371432632);
    \draw[red!40, line width=0.35pt]
      (axis cs:113,-0.002839781565963876) -- (axis cs:113,-0.0012742586477548414);
    \draw[red!40, line width=0.35pt]
      (axis cs:114,-0.0034832994973056195) -- (axis cs:114,-0.0011611086807663041);
    \draw[red!40, line width=0.35pt]
      (axis cs:115,-0.008664701819716068) -- (axis cs:115,0.007813012231567594);
    \draw[red!40, line width=0.35pt]
      (axis cs:116,0.000792401417290116) -- (axis cs:116,0.001278023568010901);
    \draw[red!40, line width=0.35pt]
      (axis cs:117,-0.004579851170080039) -- (axis cs:117,-0.00018193162614835862);
    \draw[red!40, line width=0.35pt]
      (axis cs:118,0.00014386019801163215) -- (axis cs:118,0.0019671575823686513);
    \draw[red!40, line width=0.35pt]
      (axis cs:119,0.0007309188573451919) -- (axis cs:119,0.0013173538134722217);
    \draw[red!40, line width=0.35pt]
      (axis cs:120,-0.006146244117890748) -- (axis cs:120,-0.003995835061503122);
    \draw[red!40, line width=0.35pt]
      (axis cs:121,0.0008659664851059659) -- (axis cs:121,0.0016105550373898586);
    \draw[red!40, line width=0.35pt]
      (axis cs:122,0.00038844761150058724) -- (axis cs:122,0.0008412505427346456);
    \draw[red!40, line width=0.35pt]
      (axis cs:123,0.00048269840678248215) -- (axis cs:123,0.0012188868561289047);
    \draw[red!40, line width=0.35pt]
      (axis cs:124,0.0013296490573984376) -- (axis cs:124,0.0021641259391672195);
    \draw[red!40, line width=0.35pt]
      (axis cs:125,0.002368155973830495) -- (axis cs:125,0.004358226086830424);
    \draw[red!40, line width=0.35pt]
      (axis cs:126,0.0013095672133968336) -- (axis cs:126,0.0020524756585635854);
    \draw[red!40, line width=0.35pt]
      (axis cs:127,0.0008865029888811805) -- (axis cs:127,0.0016470672686117842);
    \draw[red!40, line width=0.35pt]
      (axis cs:128,0.011743566873831554) -- (axis cs:128,0.014444033774682527);
    \draw[red!40, line width=0.35pt]
      (axis cs:129,-0.00024081449088038926) -- (axis cs:129,0.0006144901954718484);
    \draw[red!40, line width=0.35pt]
      (axis cs:130,0.000503459818563527) -- (axis cs:130,0.0008841961210384014);
    \draw[red!40, line width=0.35pt]
      (axis cs:131,-0.00022962392083221488) -- (axis cs:131,0.0008937159735447507);
    \draw[red!40, line width=0.35pt]
      (axis cs:132,-0.0025770148192042442) -- (axis cs:132,-0.0016555741927333216);
    \draw[red!40, line width=0.35pt]
      (axis cs:133,-0.00011770911197505273) -- (axis cs:133,0.00025119498827721397);
    \draw[red!40, line width=0.35pt]
      (axis cs:134,0.0001617715512616302) -- (axis cs:134,0.0006861073077685132);
    \draw[red!40, line width=0.35pt]
      (axis cs:135,0.0004544348220313977) -- (axis cs:135,0.0008345086431518255);
    \draw[red!40, line width=0.35pt]
      (axis cs:136,0.0003413568667862051) -- (axis cs:136,0.00060655443057242);
    \draw[red!40, line width=0.35pt]
      (axis cs:137,-0.00040097674581607485) -- (axis cs:137,7.058889976828051e-05);
    \draw[red!40, line width=0.35pt]
      (axis cs:138,0.00041221261867112644) -- (axis cs:138,0.0009630471718785579);
    \draw[red!40, line width=0.35pt]
      (axis cs:139,0.0013270387231938351) -- (axis cs:139,0.003461158169633201);
    \draw[red!40, line width=0.35pt]
      (axis cs:140,-0.012920309087236775) -- (axis cs:140,-0.0031680049447742886);
    \draw[red!40, line width=0.35pt]
      (axis cs:141,-0.00550760586732271) -- (axis cs:141,0.00034641730185564195);
    \draw[red!40, line width=0.35pt]
      (axis cs:142,-0.0006389493426651492) -- (axis cs:142,0.00016975677681797013);
    \draw[red!40, line width=0.35pt]
      (axis cs:143,0.00033824220244098053) -- (axis cs:143,0.0006630371207840863);
    \draw[red!40, line width=0.35pt]
      (axis cs:144,0.0028850608965265688) -- (axis cs:144,0.005448458248298933);
    \draw[red!40, line width=0.35pt]
      (axis cs:145,0.00037680230856933233) -- (axis cs:145,0.0008972358101778765);
    \draw[red!40, line width=0.35pt]
      (axis cs:146,-0.00019859880149752135) -- (axis cs:146,0.0002094593554192618);
    \draw[red!40, line width=0.35pt]
      (axis cs:147,0.0003315591528370583) -- (axis cs:147,0.0004624686243305261);
    \draw[red!40, line width=0.35pt]
      (axis cs:148,-0.0020762019553602503) -- (axis cs:148,0.00021839502498158107);
    \draw[red!40, line width=0.35pt]
      (axis cs:149,-0.00029748389310624236) -- (axis cs:149,0.00017285240923138273);
    \draw[red!40, line width=0.35pt]
      (axis cs:150,-0.0006300921175798911) -- (axis cs:150,-3.500080021060848e-05);
    \draw[red!40, line width=0.35pt]
      (axis cs:151,-0.001063750694915751) -- (axis cs:151,0.004566456853675245);
    \draw[red!40, line width=0.35pt]
      (axis cs:152,-0.0006726629337689873) -- (axis cs:152,-3.168009340316218e-05);
    \draw[red!40, line width=0.35pt]
      (axis cs:153,0.0002150228651572803) -- (axis cs:153,0.0005769393898129191);
    \draw[red!40, line width=0.35pt]
      (axis cs:154,-0.00018659496135108836) -- (axis cs:154,0.00016420419765872352);
    \draw[red!40, line width=0.35pt]
      (axis cs:155,1.1824300878598743e-05) -- (axis cs:155,0.00035098118278834045);
    \draw[red!40, line width=0.35pt]
      (axis cs:156,0.0007835057628412392) -- (axis cs:156,0.0009494470608332604);
    \draw[red!40, line width=0.35pt]
      (axis cs:157,-0.00011569654015022205) -- (axis cs:157,0.00019741776732258755);
    \draw[red!40, line width=0.35pt]
      (axis cs:158,-0.0053790989657320765) -- (axis cs:158,0.013763667567256128);
    \draw[red!40, line width=0.35pt]
      (axis cs:159,-0.0008694650497232076) -- (axis cs:159,0.000919571693156815);
    \draw[red!40, line width=0.35pt]
      (axis cs:160,0.00047963707657049126) -- (axis cs:160,0.0009093539632533257);
    \draw[red!40, line width=0.35pt]
      (axis cs:161,0.00015933108292628294) -- (axis cs:161,0.00052573028471307);
    \draw[red!40, line width=0.35pt]
      (axis cs:162,-0.0003563562403804054) -- (axis cs:162,-9.318115282234735e-06);
    \draw[red!40, line width=0.35pt]
      (axis cs:163,0.0012146908558328012) -- (axis cs:163,0.0019942759687670668);
    \draw[red!40, line width=0.35pt]
      (axis cs:164,-6.660601736878282e-06) -- (axis cs:164,0.00038290050786357095);
    \draw[red!40, line width=0.35pt]
      (axis cs:165,-0.000683213989192253) -- (axis cs:165,-0.00020108015130271205);
    \draw[red!40, line width=0.35pt]
      (axis cs:166,-0.00013762620335622787) -- (axis cs:166,0.0019969485282403054);
    \draw[red!40, line width=0.35pt]
      (axis cs:167,-0.003214218206606039) -- (axis cs:167,-0.002093853177910363);
    \draw[red!40, line width=0.35pt]
      (axis cs:168,-0.00044097873745629283) -- (axis cs:168,-0.0001046634715594989);
    \draw[red!40, line width=0.35pt]
      (axis cs:169,-8.734476769841682e-05) -- (axis cs:169,0.00034074532533186127);
    \draw[red!40, line width=0.35pt]
      (axis cs:170,-0.0011528951963268586) -- (axis cs:170,0.0004979480270240894);
    \draw[red!40, line width=0.35pt]
      (axis cs:171,-0.0025750435243214443) -- (axis cs:171,-0.0013817930848134831);
    \draw[red!40, line width=0.35pt]
      (axis cs:172,-2.6809526530871922e-05) -- (axis cs:172,6.599197349797285e-05);
    \draw[red!40, line width=0.35pt]
      (axis cs:173,-7.31437564231687e-05) -- (axis cs:173,-1.2987244866089281e-05);
    \draw[red!40, line width=0.35pt]
      (axis cs:174,-0.0003252534045505535) -- (axis cs:174,-3.364509373666595e-06);
    \draw[red!40, line width=0.35pt]
      (axis cs:175,-9.05312940989099e-05) -- (axis cs:175,-4.785248781725731e-05);
    \draw[red!40, line width=0.35pt]
      (axis cs:176,-6.020602559125546e-05) -- (axis cs:176,0.00013399062154559587);
    \draw[red!40, line width=0.35pt]
      (axis cs:177,0.0) -- (axis cs:177,0.0);
    \draw[red!40, line width=0.35pt]
      (axis cs:178,0.0) -- (axis cs:178,0.0);
    \draw[red!40, line width=0.35pt]
      (axis cs:179,0.0) -- (axis cs:179,0.0);
    \draw[red!40, line width=0.35pt]
      (axis cs:180,-0.0007087227471396017) -- (axis cs:180,0.0018317569610125217);
    \draw[red!40, line width=0.35pt]
      (axis cs:181,0.0) -- (axis cs:181,0.0);
    \draw[red!40, line width=0.35pt]
      (axis cs:182,0.0) -- (axis cs:182,0.0);
    \draw[red!40, line width=0.35pt]
      (axis cs:183,0.0) -- (axis cs:183,0.0);
    \draw[red!40, line width=0.35pt]
      (axis cs:184,0.0) -- (axis cs:184,0.0);
    \draw[red!40, line width=0.35pt]
      (axis cs:185,0.0) -- (axis cs:185,0.0);
    \draw[red!40, line width=0.35pt]
      (axis cs:186,0.0) -- (axis cs:186,0.0);
    \draw[red!40, line width=0.35pt]
      (axis cs:187,0.0) -- (axis cs:187,0.0);
    \draw[red!40, line width=0.35pt]
      (axis cs:188,0.0) -- (axis cs:188,0.0);
    \draw[red!40, line width=0.35pt]
      (axis cs:189,-0.00857784236106851) -- (axis cs:189,0.003318842036394677);
  }
  \begin{tikzpicture}
    \begin{axis}[
        width=0.96\textwidth,
        height=0.32\textwidth,
        xmin=1, xmax=189,
        ymin=-0.16, ymax=0.78,
        xtick={1,32,64,96,128,160,189},
        y coord trafo/.code={\pgfmathparse{ifthenelse(abs(#1)<=0.002,
          #1/0.002,
          sign(#1)*(1+ln(max(abs(#1),0.002)/0.002)/ln(10)))}},
        y coord inv trafo/.code={\pgfmathparse{ifthenelse(abs(#1)<=1,
          0.002*#1,
          sign(#1)*0.002*exp((abs(#1)-1)*ln(10)))}},
        ytick={-0.1,-0.01,-0.002,0,0.002,0.01,0.1,0.5},
        yticklabels={$-0.1$,$-0.01$,$-0.002$,0,0.002,0.01,0.1,0.5},
        xlabel={Feature-group rank by Individual-Feature Permutation $\Delta$AUC},
        ylabel={Relative $\Delta$AUC (\%; symlog)},
        tick label style={font=\scriptsize},
        label style={font=\small},
        axis line style={draw=black!55},
        tick style={draw=black!55},
        grid=major,
        grid style={draw=black!8},
        legend columns=3,
        legend style={
          at={(0.985,0.97)},
          anchor=north east,
          font=\scriptsize,
          draw=black!25,
          fill=white,
          inner sep=2pt,
          column sep=5pt,
        },
      ]
      \shapleycibars
      \addplot[black!45, dashed, line width=0.5pt, forget plot]
        coordinates {(1,0) (189,0)};
      \addplot[
        color=black!45,
        densely dashed,
        line width=0.75pt,
        mark=none,
      ] table[x=rank,y=individual] {\individualshapleydata};
      \addlegendentry{Individual-Feature Permutation}
      \addplot[
        color=red!70!black,
        solid,
        line width=1.1pt,
        mark=none,
      ] table[x=rank,y=shapley] {\individualshapleydata};
      \addlegendentry{20-order mean marginal}
      \addlegendimage{
        only marks,
        mark=|,
        mark size=3pt,
        color=red!40,
        line width=0.35pt,
      }
      \addlegendentry{95\% CI}
    \end{axis}
  \end{tikzpicture}
  \caption{Closeup AUC effects for 189 feature groups, ranked by
    Individual-Feature Permutation relative $\Delta$AUC. The dashed curve shows
    individual permutation; the solid curve and bars show the mean and 95\%
    confidence interval over 20 randomized addition orders, computed as
    $t_{9}$ intervals over the ten reverse-pair means. The vertical axis
    uses a symmetric-logarithmic scale.}
  \label{fig:individual-feature-shapley}
  \Description{A line plot with feature-group rank from 1 to 189 on the
    horizontal axis and relative Closeup AUC change on a symmetric-logarithmic
    vertical axis. A muted dashed line shows Individual-Feature Permutation
    effects. A prominent solid red line
    shows the corresponding order-averaged marginal effects, accompanied
    by light vertical 95 percent confidence-interval bars, and crosses below
    zero for some feature groups.}
\end{figure*}

\subsection{Grouped perturbation validation}
\label{subsec:grouped-evaluation}

We use 1.5M matched evaluation examples and measure AUC for the Closeup head across 189 feature groups that passed the non-negative Individual-Feature Permutation gate. The individual, fixed-order grouped, and order-averaged estimates use the same examples.

We then jointly perturb nested candidate sets before retraining. For set $G_k$, let
$D_k=100(\mathrm{AUC}(G_k)-\mathrm{AUC}_0)/\mathrm{AUC}_0$ and
$m_k=D_k-D_{k-1}$, with $D_0=0$. Because higher AUC is better, a positive
$m_k$ indicates less model reliance, whereas a negative $m_k$ reveals greater
joint reliance. We order groups from least to most important by their
Individual-Feature Permutation score. Pooled inference keeps the marginal cost
of each additional group low.

We find 68 of 176 nonzero fixed-order marginals ($38.6\%$) were negative.
These negative marginals show why the joint gate is necessary: feature groups
that pass individually can reveal model reliance when evaluated in a shared
perturbation context.
Without making use of grouped perturbation, we might have expected
the effect on AUC of removing these features to be the sum of the
Individual-Feature Permutation deltas: $6.554\%$.
However, not only is the grouped effect substantially smaller: $3.084\%$,
the interactions between features are sufficiently negative that the
peak cumulative delta of $3.899\%$ occurs when only 64 of the candidates
which pass the individual gate have been removed.

\stagelead{Order-averaged marginal contributions.} The fixed cumulative path
measures one addition order. To evaluate sensitivity to coalition context, we
average each group's marginal contribution across 20 reverse-paired randomized
orders, yielding an order-averaged marginal contribution for each of the same
189 groups. This is the permutation-sampling estimator \cite{castro2009} of the
Shapley value \cite{shapley1953}; we report it as an order-sensitivity
diagnostic rather than an attribution method, since 20 orders cannot resolve
individual Shapley values across 189 groups.
Reverse pairing lowers weighted attribution variance by $83$--$90\%$ relative
to independent orders. Because this requires 20 cumulative passes rather than
the funnel's one, it is not a production stage.
Figure~\ref{fig:individual-feature-shapley} shows the resulting per-group
estimates and uncertainty.

The group total is well determined while its per-group decomposition is not. The order-averaged attributions sum to $3.079\%$, against the fixed path's $3.084\%$ endpoint, with a standard error of $0.008\%$ ($0.26\%$ of the total). Within a single order the marginals telescope to the full-coalition effect exactly, so the two totals estimate the same quantity and differ only because each path re-measures it under independent permutation draws; the reported standard error is that measurement noise, not variance across orders. Attribution error cancels in the sum, so the composed set is estimated an order of magnitude more precisely than a typical member.

Two structural effects explain why per-group shares resist identification. The
model reads several features in more than one slot, and the order-averaged estimates assign
opposing signs to the two copies: one embedding group scores $+0.080\%$ in one
slot and $-0.050\%$ in another, and an impression-count group $+0.049\%$
against $-0.002\%$. Since both copies are permuted together, no split of their
shared signal is preferred, and only the joint effect is well defined. Reliance
also differs by head, sometimes in direction. For the same 189 groups the
measured full-coalition effect is removable on Closeup AUC and PR-AUC and on
Clickthrough AUC, but load-bearing on Clickthrough PR-AUC: the four attribution
sums are $+3.079\%$, $+2.348\%$, $+1.996\%$, and $-3.769\%$, respectively.
Individual groups swing much further: one
content-type group attributes $-0.042\%$ on Closeup AUC, where the model relies
on it, and $+8.385\%$ on Clickthrough AUC, where it is removable. A per-feature
screen on one head cannot see either effect.

Individual-Feature Permutation deltas also overcount additive value. Closeup retains only
$47$--$48\%$ of their sum once the full coalition is
measured: $3.079\%$ versus $6.554\%$ for AUC and $2.348\%$ versus $4.909\%$ for
PR-AUC. In addition, Clickthrough AUC reverses sign outright, from $-10.698\%$
summed to $+1.996\%$ measured. Value saturates with coalition size as well. The
mean Closeup AUC marginal falls from $0.0327\%$ in positions 1--38 to $0.0030\%$
in positions 152--189, a difference of $-0.0297\%$ with $95\%$ CI
$[-0.0358, -0.0236]$. Reverse pairing matches the two bins on
Individual-Feature Permutation score,
so this decline reflects position in the order and not feature strength. It
corroborates the fixed cumulative path's peak and subsequent decline.

As Figure~\ref{fig:individual-feature-shapley} shows, order averaging broadly
preserves the top of the Individual-Feature Permutation ranking but not the zero
threshold, with disagreements concentrated in the tail. Across individual
orders, fewer than one-third of candidates retain a positive marginal in every
order, so agreement in the mean does not imply stable per-feature effects.

These results support FORTE's staged design: per-feature estimates screen and
rank candidates, while grouped perturbation measures the composed candidate set
before retraining. Estimates from either a single permutation or 20 coalition
orders supported relative ranking but were not sufficiently stable to serve as
removal criteria: across the 176 candidates with a positive
Individual-Feature Permutation score, $2{,}496$ of $3{,}520$ order-specific
marginals ($70.9\%$) retained that direction, an average of $14.2$ of 20 orders
per group (median $16$), and only $53$ candidates retained it in all 20. The
composed-set effect was estimated more precisely, and grouped perturbation
measures it directly in one pooled pass. Order averaging
characterizes order sensitivity but does not replace grouped validation; a 10-order run reproduced the
20-order ranking closely ($0.975$ rank correlation and $95\%$ top-20 overlap).

\section{Industrial Applications}
\label{sec:applications}

FORTE is deployed across Pinterest's production model fleet. We examine two in
depth: a ranking model and a Search Relevance model. Results are reported in each
model's own section. The two models differ in architecture affecting the impact of the deprecations. The Search Relevance model deprecated a larger share of its features yet realized a small parameter reduction because its cross-encoder backbone holds orders of magnitude more parameters.

\subsection{Search Relevance}
\label{subsec:searchrel}

The Search Relevance application uses \textit{segmented evaluation}: a
post-inference filter applied at the metric layer that partitions the evaluation
set into subpopulations and computes importance scores per segment. Because
segmentation is a metric-layer filter, segmented importance reuses the same
inference pipeline, perturbation logic, and aggregation framework as global
importance. Search Relevance serves multilingual traffic, so the segments of
interest are US and non-US markets, and we use segmented A/B results to confirm
that deprecated features did not carry localized signal for non-English
queries.

The Search Relevance model \cite{searchrelevance2024} predicts the relevance of
a candidate Pin to a user's search query on a five-point ordinal scale
(L1--L5); the score is consumed alongside engagement predictions to determine
the final ranking on Pinterest Search.

The production system is a real-time servable student network distilled from a
cross-encoder LLM teacher fine-tuned from a multilingual LLM backbone, ingesting
a heterogeneous feature set spanning query-level signals, Pin-level signals
(text, graph, and visual embeddings), and query--Pin interaction features; that
work describes the full architecture, feature catalog, and distillation
pipeline. The checkpoint that FORTE evaluates extends it with a smaller
cross-encoder layered over the feed-forward student, strengthening query--Pin
interaction modeling within the real-time serving budget. The model serves
multilingual traffic globally in real time across more than 45 languages.

\stagelead{Evaluation setup.} The combination of heterogeneous tensor types,
breadth of input signals, and tight serving budget makes Search Relevance a
natural target for FORTE: many features were added incrementally over multiple
model iterations, and per-feature serving cost compounds at search QPS (queries
per second). We ran FORTE on the production Search Relevance checkpoint, which
ingests 111 features, computing both permutation and ablation importance. We
transformed the
five-point relevance output into a single scalar evaluation value and computed
the signed normalized importance score from its perturbation-induced change.
The candidate gate used an importance-score threshold of $0.01$, after which
candidates were ranked jointly with a cost analysis. Cut points of 32, 48, and
59 features were chosen to span likelihood of passing A/B. The candidate filter
(low mean importance, non-positive fitted slope, and rank stability near the
bottom of the feature list) was applied across every head and lookback window.
Segments were defined as
US and non-US traffic, with segmented results aggregated separately to confirm
neutrality across the multilingual scope.

\begin{table}[t]
  \centering
  \small
  \caption{Search Relevance: offline results and deployment outcomes for the three nested
  candidate sets, one treatment arm each. All three cleared every offline gate;
  only the online guardrails separated them. Efficiency gains are reported for the
  launched 48-feature configuration.}
  \label{tab:searchrel-results}
  \setlength{\tabcolsep}{2pt}
  \begin{tabular*}{\columnwidth}{@{\extracolsep{\fill}}rcc l cc@{}}
    \toprule
    \shortstack{\textbf{Features}\\\textbf{removed}}
      & \shortstack{\textbf{Teacher}\\\textbf{acc.\ $\Delta$}}
      & \shortstack{\textbf{Human}\\\textbf{acc.\ $\Delta$}}
      & \textbf{Outcome}
      & \shortstack{\textbf{Trainable}\\\textbf{params.\ $\Delta$}}
      & \shortstack{\textbf{Serving}\\\textbf{throughput $\Delta$}} \\
    \midrule
    32 & $-0.0010$ & $-0.0012$ & Superseded
      & \textemdash & \textemdash \\
    48 & $-0.0009$ & $+0.0042$ & \emph{Launched}$^*$
      & $-2.3\%$ & $+2.0\%$ \\
    59 & $-0.0011$ & $-0.0012$ & Rejected
      & \textemdash & \textemdash \\
    \bottomrule
  \end{tabular*}
  \par\vspace{2pt}
  \raggedright $^*$The launched configuration also realized network and storage savings.
\end{table}

\stagelead{Results.} The candidate filter surfaced 59 features for
deprecation, drawn from across feature types and entity types (text-match
features, some Pin and query embeddings, and several sparse features). Grouped
ablation of the candidate sets produced an offline 5-scale accuracy change of
$0.1\%$ relative to the production baseline. We then retrained the student model
with the candidate features removed and evaluated against the production
baseline on the same human-annotated relevance test set used in
that work: 5-scale accuracy fell within the offline-neutral
band for this model, for all three candidate sets.

The retrained models were then allocated traffic as treatment arms and A/B tested
against production using relevance measurements obtained on queries from search
traffic. Table~\ref{tab:searchrel-results} reports each arm. We observed no
statistically significant regression in search feed relevance for any arm;
this evaluation uses sDCG@25 from an LLM-as-a-judge labeling system
\cite{wang2025llmrelevance}. Segment-level results across US and non-US markets
were also neutral, indicating that the deprecated features did not carry
localized signal for non-English queries despite the model's multilingual scope.

The offline gates did not separate the three candidate sets: teacher-test accuracy
moved by at most $0.0011$ in every case, and teacher-test RMSE improved in every
case ($-0.0003$, $-0.0037$, $-0.0044$). The online stage did. Search processing
time improved significantly for all three arms; for the 48-feature arm, page
wait time decreased by $0.1696\%$ ($p{=}0.0012$). The most aggressive
59-feature arm triggered a regression in a secondary guardrail and was not
launched; the 48-feature set was selected as the safer operating point. Shipping
the smaller model reduced the trainable parameter count by $2.3\%$, increased
serving throughput by $2\%$, and reduced network ingress and egress for our
model serving system, as well as enabling savings on storage and KV feature
store costs.

\subsection{Pin Ranker}
\label{subsec:pinranker}

The Pin Ranker is a multi-head L2 ranking model powering a visual-search--like
surface that serves roughly half of Pinterest traffic. Given a
(\text{user}, \text{query Pin}, \text{candidate Pin}) triplet, the production
checkpoint predicts closeups, clickthroughs, repins, and downloads. It also includes a
revisitation-based retention head \cite{jiang2026save}.

The architecture combines a transformer-based user-sequence encoder, including
lifelong user-history features \cite{xia2025transact}, with a DCNv2
feature-crossing stage and an MMoE module that produces task-specific engagement
scores \cite{ebrahimi2026warmer}. It has hundreds of millions of dense parameters
spanning multi-modal content, numerical, categorical, and user-behavior
features. These components enable the system to model rich dependencies among
user intent, history, content semantics, and candidate relevance.

\stagelead{Evaluation setup.} We ran FORTE on the production Pin Ranker
checkpoint, computing the signed normalized change in PR-AUC under permutation
across all output heads. Following the deprecation pilot, we now run feature
importance daily on the latest production model. This ongoing practice provides a
consistent benchmark for the team's feature utilization and assists in
monitoring feature health and adoption over time. Temporal importance was
combined with a cost analysis; an audit identified two nested feature sets, of
2 and 8 features, that yielded substantial cost savings at weak measured
reliance.

\stagelead{Results.} FORTE identified eight features as deprecation candidates:
each had a non-negative permutation importance score against PR-AUC across
all output heads, indicating no predictive signal the model used. We retrained
the model without these features. Both candidate sets cleared every gate.
Retraining produced an absolute pairwise-accuracy change of $+0.0001$ for
both sets, while the absolute $\Delta$NDCG@3 was $-0.0015$ and $-0.0010$ for the
2- and 8-feature sets, respectively. These results are averaged across output
heads and remained within the offline-neutral band. As treatment arms over a
three-week
online A/B test, neither showed a statistically significant regression in
clickthroughs, closeups, or repins. The 8-feature set is a superset of the 2-feature
set, so with both arms neutral the larger set was launched and the smaller one
superseded.

After launch, FORTE's final lineage-and-deprecation stage
(Figure~\ref{fig:funnel}) queried the model registry for remaining consumers of
the eight retired features. None remained; after owner review, the corresponding
dataset columns were removed and their feature-materialization jobs
decommissioned.

Table~\ref{tab:pinranker-efficiency} reports the offline outcomes and resulting
efficiency gains.
The parameter, throughput, and inference-cost gains follow directly from the
smaller model, and together reduced annual infrastructure cost by hundreds of
thousands of dollars. Completing deprecation across the feature lineage added
dataset-storage and training-loop savings beyond those available from a
single-model launch.

\begin{table}[t]
  \centering
  \footnotesize
  \caption{Pin Ranker: Retraining results for the two nested candidate
  sets and efficiency gains from launching the 8-feature removal. Throughput is
  measured on a CUDA Graph benchmark.}
  \label{tab:pinranker-efficiency}
  \setlength{\tabcolsep}{1pt}
  \begin{tabular*}{\columnwidth}{@{\extracolsep{\fill}}rcc l ccc@{}}
    \toprule
    \shortstack{\textbf{Features}\\\textbf{removed}}
      & \shortstack{\textbf{Pairwise}\\\textbf{Accuracy}\\\textbf{$\Delta^*$}}
      & \shortstack{\textbf{NDCG@3}\\\textbf{$\Delta^*$}}
      & \textbf{Outcome}
      & \shortstack{\textbf{Trainable}\\\textbf{params.\ $\Delta$}}
      & \shortstack{\textbf{Serving}\\\textbf{throughput}\\\textbf{$\Delta$}}
      & \shortstack{\textbf{Inference}\\\textbf{cost $\Delta$}} \\
    \midrule
    2 & $+0.0001$ & $-0.0015$ & Superseded
      & \textemdash & \textemdash & \textemdash \\
    8 & $+0.0001$ & $-0.0010$ & \emph{Launched}
      & $-29.9\%$ & $+19.8$--$25.5\%$ & ${\sim}-10\%$ \\
    \bottomrule
  \end{tabular*}
  \par\vspace{2pt}
  \raggedright $^*$Retraining results are absolute changes averaged across
  output heads.
\end{table}

\stagelead{Diagnostic use.} Beyond its utility in identifying candidates for
feature deprecation, the feature-importance report serves as a critical
diagnostic mechanism for explaining why a deep learning architectural change had
its observed effect. For example, in the residual-connection layer design for
non-engagement features in the ranker model, the objective is to reduce the
model's over-reliance on engagement-biased signals and elevate content and
quality signals in the decision process. FORTE quantifies the over-reliance
directly: in the baseline ranker, the ratio of $\Delta$PR-AUC from ablating
non-historical features versus historical features sits well below $100\%$,
confirming that the model leans disproportionately on historical signals
\cite[Fig.~2(a)]{ebrahimi2026warmer}. The per-feature $\Delta$PR-AUC
distribution under the residual variant shifts more negative for non-historical
features relative to the baseline, indicating that the residual connection
raises the model's reliance on those features as designed
\cite[Fig.~6(a)]{ebrahimi2026warmer}. Reading the full importance ranking turns
FORTE into a tool for implementation validation, not only post-hoc
interpretation.

\section{Discussion and Conclusion}
\label{sec:discussion}

FORTE combines scalable daily feature evaluation with validation and lineage tracking to support production feature deprecation and infrastructure savings.

\subsection{Limitations}

FORTE prioritizes scalable feature-deprecation decisions over full causal
attribution. This choice creates limitations. Marginal permutation does not
achieve conditional PFI \cite{strobl2008}, so correlated features can receive
inflated or deflated importance scores. FORTE also measures first-order
importance for individual features or manually configured groups, which can miss
features that matter only through interactions.

Temporal feature importance is computed against a model held fixed within each
analysis window,
making changes interpretable as data drift. New windows refresh to the current
production checkpoint, so trend comparisons are scoped to a single checkpoint.

The two-stage downsample provides an approximate shuffling, trading exactness
for runtime; we do not claim it is equivalent to a uniform shuffle.

We apply a single random permutation ($R=1$) to each batch by default. $R$ is
exposed as a hyper-parameter, and each additional repeat multiplies the number
of perturbation experiments; batch size and total dataset size also bear on the
resulting variance. Section~\ref{subsec:single-permutation-proof} provides a finite-sample
analysis and empirical stability measurements supporting this operating choice;
the measurements cover two heads of one production model and do not establish
the choice for every model or metric.

We share two examples on Pinterest models; the efficiency findings in
particular depend on model characteristics and should not be read as
architecture-independent.

The production workflow is designed around these limitations. Candidate
features are selected conservatively across all heads, evaluated jointly through
grouped perturbation, removed through retraining, and validated through online A/B
testing before removal. Future work will focus on automatic interaction
discovery, conditional importance estimates for correlated features, and alerts
for feature failures and data-quality changes over temporal importance series.

\section*{Ethical Considerations}

A feature removal that is neutral in aggregate can still degrade quality for
subpopulations the aggregate does not reflect. FORTE mitigates this with
metric-layer segmented evaluation, applied in this work as US versus non-US
traffic. Two coarse geographic segments cannot resolve effects across the more
than 45 languages the Search Relevance model serves, and can conceal
within-segment harm; this is one mitigation with limited resolution, not a
fairness result. Deployments should monitor at the granularity their traffic
warrants, whether per-language, per-geography, or per-traffic-slice, chosen for
the population actually served.

We note a risk the framework itself introduces. Deprecation errors surface as
aggregate metric regressions, which is precisely the measurement that hides
distributional harm; making removal cheap and routine increases the throughput
of a decision whose failure mode is unevenly distributed.

Removing unused features reduces the compute, storage, and serving footprint of
a model. Our measured effects are a reduced parameter count, higher serving
throughput, and lower storage footprint; removal is expected to reduce energy
use as a consequence, but we did not measure power, accelerator-hours, or
emissions and make no quantitative energy claim.

\bibliographystyle{ACM-Reference-Format}
\bibliography{forte}

\end{document}